\documentclass[journal]{IEEEtranTIE}
\usepackage{amsfonts}
\usepackage{amssymb}
\usepackage[cmex10]{amsmath}
\usepackage{graphicx}
\usepackage{lipsum}
\usepackage{subfigure}
\usepackage{verbatim}
\usepackage{cite}
\usepackage{fancyhdr}
\usepackage{tabularx}
\usepackage{booktabs}
\usepackage[ruled]{algorithm2e}
\usepackage{multirow}
\usepackage{hyperref}
\usepackage{bm}
\usepackage{stfloats}
\usepackage{bm}
\usepackage{enumerate}
\usepackage{wasysym}
\usepackage{bm}
\usepackage{amsmath}
\usepackage{amsthm}
\usepackage{fancyhdr}
\usepackage{amsfonts}
\usepackage{mathrsfs}
\usepackage{pifont}
\usepackage{amssymb}
\usepackage{verbatim}
\usepackage{upgreek}
\usepackage{graphicx}
\usepackage{epstopdf}
\usepackage{subfigure}
\usepackage{color}
\usepackage{amsmath}
\usepackage{algorithmic}
\newcolumntype{C}{>{\Centering\arraybackslash}X} 

\title{Quadrotor Stabilization with Safety Guarantees: A Universal Formula Approach}
\author{Ming Li$\,\,\emph{Student Member, IEEE}$, Zhiyong Sun$ \,\,\emph{Member, IEEE}$, and Siep~Weiland} 
\begin{document}
\newtheorem{Thm}{\textbf{Theorem}}
\newtheorem{Lem}{\textbf{Lemma}}
\newtheorem{Def}{\textbf{Definition}}
\newtheorem{Rem}{\textbf{Remark}}
\newtheorem{Exam}{\textbf{Example}}
\newtheorem{Sup}{\textbf{Assumption}}
\newtheorem{Cor}{\textbf{Collary}}
\newtheorem{Asum}{\textbf{Assumption}}
\newtheorem{Expl}{\textbf{Explanation}}
\newtheorem{Prop}{\textbf{Proposition}}
\newtheorem{Rmk}{\textbf{Remark}}
\maketitle
\begin{abstract}
Safe stabilization is a significant challenge for quadrotors, which involves reaching a goal position while avoiding obstacles. Most of the existing solutions for this problem rely on optimization-based methods, demanding substantial onboard computational resources. This paper introduces a novel approach to address this issue and provides a solution that offers fast computational capabilities tailored for onboard execution. Drawing inspiration from Sontag's universal formula, we propose an analytical control strategy that incorporates the conditions of control Lyapunov functions (CLFs) and control barrier functions (CBFs), effectively avoiding the need for solving optimization problems onboard. Moreover, we extend our approach by incorporating the concepts of input-to-state stability (ISS) and input-to-state safety (ISSf), enhancing the universal formula's capacity to effectively manage disturbances. Furthermore, we present a projection-based approach to ensure that the universal formula remains effective even when faced with control input constraints. The basic idea of this approach is to project the control input derived from the universal formula onto the closest point within the control input domain. Through comprehensive simulations and experimental results, we validate the efficacy and highlight the advantages of our methodology.
\let\thefootnote\relax\footnotetext{This work was supported in part by a starting grant from Eindhoven Artificial Intelligence Systems Institute (EAISI), The Netherlands. \emph{(Corresponding author: Zhiyong~Sun.)}

The authors are with the Department of Electrical Engineering, Eindhoven University of Technology, and also with the Eindhoven Artificial Intelligence Systems Institute, PO Box 513, Eindhoven 5600 MB, The Netherlands. 
{\tt\small \{ m.li3, z.sun, s.weiland \}@tue.nl}
}
\end{abstract}
\section{Introduction}
A quadrotor is comprised of helicopters with four rotors, establishing a distinctive and favored category within the realm of unmanned aerial vehicle platforms. Renowned for its cost-efficient construction, ease of maintenance, and remarkable maneuverability, this platform has garnered substantial attention with numerous applications across various fields such as environmental monitoring ~\cite{UAV_Environment}, agriculture ~\cite{UAV_Agriculture}, and search and rescue operations ~\cite{UAV_Rescue}. However, despite boasting six Degrees of Freedom (DOFs), a quadrotor operates with only four independent thrust forces, rendering it an underactuated system ~\cite{Underactuated_Reasoning}. This inherent characteristic, coupled with its intricate nonlinear dynamics, strong coupling, and multi-variable actuation, renders the task of quadrotor control exceedingly challenging.

Safe stabilization, which entails directing the system towards a desired state while avoiding entering into unsafe regions, presents a persistent and significant challenge within quadrotor applications~\cite{quadrotor_safe_stabilization1,quadrotor_safe_stabilization2,quadrotor_safe_stabilization3}. In addressing this issue, the classic stabilization technique of Control Lyapunov Functions (CLFs)~\cite{Global_Stabilizable} and the more recent safety-critical control method known as Control Barrier Functions (CBFs)~\cite{Zeroing_CBF} are frequently combined via different techniques. For example, in~\cite{CLBF}, a CLF and a CBF are combined into a control Lyapunov barrier function (CLBF), and then a feedback control law is constructed using Sontag's universal formula~\cite{Sontag_fromula} for safe stabilization. Although this approach is commended for its simplicity, it faces challenges in meeting the existence conditions of a CLBF~\cite{Property_on_CLBF}. Besides, it may fail to guarantee stability and safety simultaneously in some proximity to the boundary of a safety set, even when the conditions of both CLFs and CBFs are satisfied~\cite{Comment_on_CLBF}. To overcome this limitation, the optimization-based methods are proposed, including nonlinear model predictive control (MPC) with state constraints~\cite{crazyflie_mpc_1,crazyflie_mpc_2}, MPC integrated with CBFs~\cite{crazyflie_mpc_CBF}, and a combination of CLF and CBF with a quadratic program (QP) formulation~\cite{CBF_Definition}. Despite the demonstrated efficacy of these methodologies, limitation arises from the necessity to solve optimization problems, rendering it infeasible to perform real-time computations directly on quadrotors with constrained computation resources. 

To handle the computation requirements for quadrotor stabilization with safety guarantees, the existing optimization-based approaches usually adopt 
a position-attitude controller architecture~\cite{Cascaded_Structure}. It combines both a planning step for generating setpoints and a real-time onboard control step, which enables a fast response to safe stabilization tasks. In this setup, a remote ground station computer is utilized to solve an MPC optimization problem and generate setpoint commands~\cite{crazyflie_mpc_1,crazyflie_mpc_2}. Subsequently, the quadrotor executes tracking tasks through a low-level onboard controller, such as PID controller. This framework has also been utilized in other studies, as seen in~\cite{quadrotor_safe_stabilization3} and~\cite{QP_Two_Layer_IMPLEMENTATION}. Compared to the approaches in~\cite{crazyflie_mpc_1,crazyflie_mpc_2}, the primary distinction of the solutions in~\cite{quadrotor_safe_stabilization3,QP_Two_Layer_IMPLEMENTATION} is that they use a CLF-CBF-QP approach to generate setpoint commands rather than MPC. However, this configuration does not represent an onboard implementation since the safe stabilization is primarily achieved on the remote computer while the quadrotor's role lies in trajectory tracking. Alternatively, some approaches involve solving nonlinear MPC directly onboard using an embedded computer, as demonstrated in \cite{On_board_MPC}. However, as onboard control requires a microprocessor with substantial computational power, it is not suitable for certain resource-constrained platforms, such as the Crazyflie~\cite{Crazyflie}. To address this challenge, researchers have explored algorithms aimed at expediting optimization processes, as discussed in~\cite{fast_speed, Fast_Embbeded}. Nevertheless, the expedition solution still falls short of delivering satisfactory performance, particularly when dealing with severely resource-constrained platforms. Furthermore, hybrid computing techniques have been proposed in~\cite{Embedded_Implementation}. These techniques involve hardware modifications to the quadrotor, allowing the implementation of the nominal stabilizing control algorithm in digital form while solving safety-critical Quadratic Programming (QP) through a dedicated analog resistor array. Due to its requirement for hardware modifications tailored to specific optimization problems, hybrid computing is impractical for some resource-intensive applications. In contrast to the aforementioned approaches, an alternative solution is to design an analytical feedback control law, namely a universal formula~\cite{universal_formula,Smooth_universal}, to achieve safe stabilization. The universal formula provides a combination of the control laws obtained from CLF and CBF conditions, which can be regarded as an analytical solution to an optimization problem with CLF and CBF constraints. Consequently, it achieves equivalent performance to the optimization-based approaches without the necessity of solving an optimization problem. Essentially, the idea is similar to the explicit MPC~\cite{Explicit_MPC}, which pre-solves the optimization problem offline, significantly reducing the online calculation. In contrast to explicit MPC, which employs a mapping table to encode the state-control relations, the universal formula relies on an analytical expression to define this relationship.

In this paper, we propose to address the quadrotor safe stabilization problem under a position-attitude controller architecture. To achieve fast computation and onboard implementation on a quadrotor, the universal formula construction strategy is studied. The main contributions of this paper are summarized as follows.

\begin{itemize}
    \item  We propose a universal formula for the position controller to generate setpoints for the attitude controller that ensure simultaneous safety and stability. Attitude control is then accomplished using a PID controller responsible for trajectory tracking. As both controller computations are executed using the quadrotor's onboard resources, we achieve a real onboard safe stabilization. Furthermore, considering that the constraints of CLF and CBF might be conflicting, the universal formulas for both compatibility and incompatibility cases are discussed. 
    \item We extend our universal formula by incorporating the theories of input-to-state stability (ISS) and input-to-state safety (ISSf), accommodating disturbances and model uncertainties. This modification yields a robust safe stabilization while maintaining fast computational speed and on-board implementation. Furthermore, when dealing with input constraints, we introduce a projection-based technique. This technique entails projecting the control input obtained from the universal formula onto the nearest point within the control input domain, thereby ensuring compatibility of the ISS and ISSf constraints.
    \item We employ both simulations and real-world experiments to validate the effectiveness and highlight the advantages of our approach. The results demonstrate that our solution exhibits significantly faster execution, which makes it well-suited for onboard implementation while maintaining a satisfying performance. 
\end{itemize}

The paper is structured as follows. In Section II, we provide an overview of CLF and exponential CBF (ECBF), and we introduce Sontag's universal formula for achieving stabilization, safety-critical control, and safe stabilization control. In Section III, we develop a universal formula for quadrotor stabilization while ensuring safety. This section adopts the position-attitude control architecture, employing the universal formula in the position controller. In Section IV, we tailor the universal formula to accommodate practical safe stabilization affected by disturbances and constrained inputs. Section V is dedicated to the development of simulations and experimental results. Finally, in Section VI, we conclude this paper.
\section{Background Preliminaries}
Consider a dynamical system that has a control-affine structure~\cite{Dynamical_model}
	\begin{equation}\label{Affine_Control_System}
	    \dot{\mathbf{x}}=\mathbf{f}(\mathbf{x})+\mathbf{g}(\mathbf{x})\mathbf{u},
	\end{equation}
where $\mathbf{x}\in\mathbb{R}^{n}$ is the system state, $\mathbf{u}\in\mathbb{R}^{m}$ is the control input, and the vector fileds $\mathbf{f}:\mathbb{R}^{n}\rightarrow\mathbb{R}^{n}$ and $\mathbf{g}:\mathbb{R}^{n}\rightarrow\mathbb{R}^{n\times m}$ are locally Lipschitz continuous.
\subsection{CLF and Universal Formulas}
\begin{Def}
(CLF~\cite{Sontag_fromula}) A continuously differentiable, positive definite, and radially unbounded function $V: \mathbb{R}^{n}\rightarrow\mathbb{R}_{+}$ is a CLF for system \eqref{Affine_Control_System} if there exists a control $\mathbf{u}\in\mathbb{R}^{m}$ satisfying
\begin{equation}\label{CLF_Condition}
\begin{aligned}
& a(\mathbf{x})+\mathbf{b}(\mathbf{x}) \mathbf{u}\leq 0
\end{aligned}
\end{equation}
for all $\mathbb{R}^{n}\backslash\{ \mathbf{0}\}$, where $a(\mathbf{x})=L_{\mathbf{f}} V(\mathbf{x})+\lambda V(\mathbf{x})$, $\lambda>0$, and $\mathbf{b}(\mathbf{x})=L_{\mathbf{g}} V(\mathbf{x})$. $L_{\mathbf{f}}$ and $L_{\mathbf{g}}$ denote the Lie derivatives along $\mathbf{f}$ and $\mathbf{g}$, respectively.
\end{Def}
The goal is to obtain a locally Lipschitz continuous feedback control law $\mathbf{u}: \mathbb{R}^{n}\rightarrow\mathbb{R}^{m}$ that satisfies the condition~\eqref{CLF_Condition} for any $\mathbb{R}^{n}\backslash\{ \mathbf{0}\}$. Toward this objective, Sontag's universal formula~\cite{Sontag_fromula} can be employed
\begin{equation}\label{CLF_Sontag_Law}
\begin{split}
\mathbf{u}_{\mathrm{Stg-CLF}}^{\star}(\mathbf{x})=\left\{\begin{array}{cc}
\mathbf{m}_{\mathrm{Stg-CLF}}^{\star}(\mathbf{x}), & \mathbf{b}(\mathbf{x}) \neq \mathbf{0}, \\
\mathbf{0}, & \mathbf{b}(\mathbf{x})=\mathbf{0},
\end{array}\right.
\end{split}
\end{equation}
where  $\mathbf{m}_{\mathrm{Stg-CLF}}^{\star}(\mathbf{x})=-\frac{a(\mathbf{x})+\kappa\sigma(\mathbf{x})}{\mathbf{b}(\mathbf{x}) \mathbf{b}(\mathbf{x})^{\top}} \mathbf{b}(\mathbf{x})^{\top}$, $\sigma(\mathbf{x})=\sqrt{a^{2}(\mathbf{x})+\phi(\mathbf{x})\|\mathbf{b}(\mathbf{x}) \|^{4}}$, and $\phi(\mathbf{x})$ is a positive semi-definite function. Note that, in contrast to the universal formula presented in~\cite{Sontag_fromula}, we introduce an additional parameter $\kappa\geq 0$ to the universal formula. This parameter provides the flexibility to adjust the convergence of stabilization.
\subsection{Exponential CBF and Universal Formulas}
Consider a closed convex set $\mathcal{C}\subset\mathbb{R}^{n}$ as the $0$-superlevel set of a  continuously differentiable function $h:\mathbb{R}^{n}\rightarrow\mathbb{R}$, which is defined as
\begin{equation}\label{Invariant_Set}
		\begin{aligned}
			\mathcal{C} & \triangleq\left\{\mathbf{x}\in \mathbb{R}^{n}: h(\mathbf{x}) \geq 0\right\}, \\
			\partial \mathcal{C} & \triangleq\left\{\mathbf{x}\in\mathbb{R}^{n}: h(\mathbf{x})=0\right\}, \\
			\operatorname{Int}(\mathcal{C}) & \triangleq\left\{\mathbf{x}\in \mathbb{R}^{n}: h(\mathbf{x})>0\right\},
		\end{aligned}
\end{equation}
where we assume that $\mathcal{C}$ is nonempty and has no isolated points, that is, $\operatorname{Int}(\mathcal{C}) \neq \emptyset$ and $\overline{\operatorname{Int}(\mathcal{C})}=\mathcal{C}$. 
\begin{Def}\label{ECBF_Conditon}
(ECBF ~\cite{ECBF_Quadrotor}) Give a set $\mathcal{C}\subset\mathbb{R}^{n}$ defined as the superlevel set of a $r$-times continuously differentiable function $h: \mathbb{R}^{n}\rightarrow\mathbb{R}$, then $h$ is an ECBF for system  \eqref{Affine_Control_System} if there exists a control input $\mathbf{u}\in\mathbb{R}^{m}$ such that 
\begin{equation}\label{ECBF}
    \begin{split}
        &c(\mathbf{x})+\mathbf{d}(\mathbf{x})\mathbf{u}\geq 0,
    \end{split}
\end{equation}
where $c(\mathbf{x})=\mathcal{L}_{\mathbf{f}}^{r}h(\mathbf{x})+\mathbf{K}^{\top}\mathcal{H}$, and $\mathbf{d}(\mathbf{x})=\mathcal{L}_{\mathbf{g}}\mathcal{L}_{\mathbf{f}}^{r-1}h(\mathbf{x})$, $\mathcal{H}=[h(\mathbf{x}),\mathcal{L}_{\mathbf{f}}h(\mathbf{x}),\cdots,\mathcal{L}_{\mathbf{f}}^{r-1}h(\mathbf{x})]^{\top}$ is the Lie derivative vector for $h(\mathbf{x})$ and $\mathbf{K}=[k_{0},k_{1},\cdots,k_{r-1}]^{\top}\in\mathbb{R}^{r}$ is the coefficient gain vector for $\mathcal{H}$.
\end{Def}
In~\cite{ECBF_Quadrotor}, it has been demonstrated that the use of the ECBF condition guarantees the forward invariance of the safety set $\mathcal{C}$ for a system with a high relative degree constraint.

Similarly, the universal control law for CBF is given as 
\begin{equation*}\label{CBF_Universal_Law}
\begin{split}
\mathbf{u}_{\mathrm{Stg-CBF}}^{\star}(\mathbf{x})=\left\{\begin{array}{cc}
\mathbf{n}_{\mathrm{Stg-CBF}}^{\star}(\mathbf{x}), & \mathbf{d}(\mathbf{x})\neq\mathbf{0},\\ \mathbf{0},&\mathbf{d}(\mathbf{x})=\mathbf{0},
\end{array}\right.
\end{split}
\end{equation*}
where $\mathbf{n}_{\mathrm{Stg-CBF}}^{\star}(\mathbf{x})=\frac{\rho\Gamma(\mathbf{x})-c(\mathbf{x})}{\mathbf{d}(\mathbf{x}) \mathbf{d}(\mathbf{x})^{\top}}\mathbf{d}(\mathbf{x})^{\top}$, $\Gamma(\mathbf{x})=\sqrt{c^{2}(\mathbf{x})+\varphi(\mathbf{x})\|\mathbf{d}(\mathbf{x}) \|^{4}}$, and $\varphi(\mathbf{x})$ is a positive semi-definite function, and the parameter $\rho\geq 0$ is employed to determine the level of conservatism in the safety guarantees.
\subsection{Generalized Universal Formula for Safe Stabilization}\label{G_Universal}
\begin{Def}
(Compatibility) Consider a collection of sets $\mathcal{S}_{i}=\{\mathbf{u}\in\mathbb{R}^{m}:p_{i}(\mathbf{x})+\mathbf{q}_{i}(\mathbf{x})\mathbf{u}\leq 0, i=1,...,M\}$, where $p_{i}(\mathbf{x})$ and $\mathbf{q}_{i}(\mathbf{x})$ for $i=1,...,M$ are state dependent parameters for these inequalities. These inequalities are compatible if, for each $\mathbf{x}\in\mathbb{R}^{n}$, there exists a $\mathbf{u}$ that satisfies all inequalities.
\end{Def}
\subsubsection{Compatible Safe Stabilization}
Firstly, we present the universal formula to address the safe stabilization problem considering the compatible scenario.
\begin{Lem}\label{Compatible_Lemma}
(\cite{universal_formula}) Assume that both $\mathbf{b}(\mathbf{x})$ and $\mathbf{d}(\mathbf{x})$ are nonzero vectors. The CLF $V(\mathbf{x})$ and ECBF $h(\mathbf{x})$ for the system~\eqref{Affine_Control_System} are compatible if and only if one of the following conditions is satisfied:
\begin{equation}\label{Compatibility_conditions}
\begin{split}
\begin{cases}
     &\frac{\|\mathbf{b}(\mathbf{x})\mathbf{d}(\mathbf{x})^{\top}\|}{\|\mathbf{b}(\mathbf{x})\|\|\mathbf{d}(\mathbf{x})\|}\neq 1,\\
     &\frac{\|\mathbf{b}(\mathbf{x})\mathbf{d}(\mathbf{x})^{\top}\|}{\|\mathbf{b}(\mathbf{x})\|\|\mathbf{d}(\mathbf{x})\|}= 1\,\,\,  \text{and}\,\,\, v(\mathbf{x})\geq 0, \\
     &\frac{\|\mathbf{b}(\mathbf{x})\mathbf{d}(\mathbf{x})^{\top}\|}{\|\mathbf{b}(\mathbf{x})\|\|\mathbf{d}(\mathbf{x})\|}= 1\,\,\, \text{and}\,\,\, w(\mathbf{x})\geq 0,
     \end{cases}
\end{split}
\end{equation}
where $w(\mathbf{x})=a(\mathbf{x})\mathbf{d}(\mathbf{x})^{\top}\mathbf{d}(\mathbf{x})-\mathbf{c}(\mathbf{x}) \mathbf{b}(\mathbf{x})^{\top} \mathbf{d}(\mathbf{x})$ and $ v(\mathbf{x})=a(\mathbf{x})\mathbf{d}(\mathbf{x})^{\top} \mathbf{b}(\mathbf{x})-\mathbf{c}(\mathbf{x})\mathbf{b}(\mathbf{x})^{\top} \mathbf{b}(\mathbf{x})$.
\end{Lem}
\begin{Thm}\label{Compatible_Qp}
(\cite{universal_formula}) Assume that the CLF $V(\mathbf{x})$ and ECBF $h(\mathbf{x})$ for the system~\eqref{Affine_Control_System} are compatible. The generalized universal formula 
\begin{equation}\label{QP_Control_Law_relaxed}
    \begin{split}
      \mathbf{u}_{\mathrm{Stg}}^{\star}(\mathbf{x})=\left\{\begin{array}{ll}
\mathbf{m}_{\mathrm{Stg-CLF}}^{\star}(\mathbf{x}),&\quad\mathbf{x}\in\mathcal{P}_{1},\\
\mathbf{n}_{\mathrm{Stg-CBF}}^{\star}(\mathbf{x}),&\quad\mathbf{x}\in\mathcal{P}_{2},\\
\mathbf{p}_{\mathrm{Stg}}^{\star}(\mathbf{x}),&\quad\mathbf{x}\in\mathcal{P}_{3},\\
\mathbf{0},&\quad\mathbf{x}\in\mathcal{P}_{4},
\end{array}\right.  
    \end{split}
\end{equation}
ensures both of the safety and stability of~\eqref{Affine_Control_System} simultaneously, where $\mathbf{m}_{\mathrm{Stg-CLF}}^{\star}(\mathbf{x})=-\frac{a(\mathbf{x})+\kappa\sigma(\mathbf{x})}{\mathbf{b}(\mathbf{x})\mathbf{b}(\mathbf{x})^{\top}} \mathbf{b}(\mathbf{x})^{\top}$, $\mathbf{n}_{\mathrm{Stg-CBF}}^{\star}(\mathbf{x})= \frac{\rho\Gamma(\mathbf{x})- c(\mathbf{x})}{\mathbf{d}(\mathbf{x})\mathbf{d}(\mathbf{x})^{\top}} \mathbf{d}(\mathbf{x})^{\top}$,  $\mathbf{p}_{\mathrm{Stg}}^{\star}=\epsilon_{2}\mathbf{m}_{\mathrm{CLF}}^{\star}(\mathbf{x})
+\epsilon_{3}\mathbf{n}_{\mathrm{CBF}}^{\star}(\mathbf{x})$, $\lambda_{2}=-\lambda_{1}\frac{\mathbf{b}(\mathbf{x})\mathbf{b}(\mathbf{x})^{\top}}{a(\mathbf{x})+\kappa\sigma(\mathbf{x})}$, $\epsilon_{3}=\lambda_{2}\frac{\mathbf{d}(\mathbf{x})\mathbf{d}(\mathbf{x})^{\top}}{\rho\Gamma(\mathbf{x})-c(\mathbf{x})}$,
\begin{equation*}\label{bar_lambda_compute}
    \begin{split}
        \begin{bmatrix}
\lambda_{1}\\
\lambda_{2}
\end{bmatrix}={\begin{bmatrix}
\mathbf{b}(\mathbf{x})\mathbf{b}(\mathbf{x})^{\top} & -\mathbf{b}(\mathbf{x})\mathbf{d}(\mathbf{x})^{\top}\\
-\mathbf{d}(\mathbf{x})\mathbf{b}(\mathbf{x})^{\top} & \mathbf{d}(\mathbf{x})\mathbf{d}(\mathbf{x})^{\top}
\end{bmatrix}}^{-1}\begin{bmatrix}
a(\mathbf{x}))+\kappa\sigma(\mathbf{x})\\
\rho\Gamma(\mathbf{x})-c(\mathbf{x})
\end{bmatrix},
    \end{split}
\end{equation*}
\begin{equation*}\label{Domain_of_sets_relaxed}
\begin{aligned}
&\mathcal{P}_{1}=\{\mathbf{x}\in\mathbb{R}^{n}|a(\mathbf{x})+\kappa\sigma(\mathbf{x})\geq 0,v(\mathbf{x})< 0\},\\
&\mathcal{P}_{2}=\{\mathbf{x}\in\mathbb{R}^{n}|c(\mathbf{x})-\rho\Gamma(\mathbf{x})\leq 0,w(\mathbf{x})< 0\},\\
&\mathcal{P}_{3}=\{\mathbf{x}\in\mathbb{R}^{n}|w(\mathbf{x})\geq 0,v(\mathbf{x})\geq 0, \\
&\mathbf{b}(\mathbf{x})^{\top} \mathbf{b}(\mathbf{x}) \mathbf{d}(\mathbf{x})^{\top} \mathbf{d}(\mathbf{x})-\mathbf{b}(\mathbf{x})^{\top} \mathbf{d}(\mathbf{x}) \mathbf{b}(\mathbf{x})^{\top} \mathbf{d}(\mathbf{x}) \neq 0\},\\
&\mathcal{P}_{4}=\{\mathbf{x}\in\mathbb{R}^{n}|a(\mathbf{x})+\kappa\sigma(\mathbf{x})< 0,c(\mathbf{x})-\rho\Gamma(\mathbf{x})> 0\}.
\end{aligned}
\end{equation*}
\end{Thm}
\begin{Rmk}\label{Universal_Interpre}
    It should be noticed that~\eqref{QP_Control_Law_relaxed} is the solution of the following optimization problem.
    \begin{equation}\label{QP_formulation}
        \begin{split}
            \min\limits_{\mathbf{u}\in\mathbb{R}^{m}}&\frac{1}{2}\|\mathbf{u}\|^{2}\\
            &a(\mathbf{x})+\mathbf{b}(\mathbf{x}) \mathbf{u}\leq -\kappa\sigma(\mathbf{x})\\
            &c(\mathbf{x})+\mathbf{d}(\mathbf{x}) \mathbf{u}\geq \rho\Gamma(\mathbf{x}).
        \end{split}
    \end{equation}
In fact, $\sigma(\mathbf{x})$ and $\Gamma(\mathbf{x})$ can be designed accordingly with certain requirements as long as they are positive semi-definite. Consequently, the universal formula can be seen as a special solution of~\eqref{QP_formulation} with $\sigma(\mathbf{x})=\sqrt{a^{2}(\mathbf{x})+\phi(\mathbf{x})\|\mathbf{b}(\mathbf{x}) \|^{4}}$ and $\Gamma(\mathbf{x})=\sqrt{c^{2}(\mathbf{x})+\varphi(\mathbf{x})\|\mathbf{d}(\mathbf{x}) \|^{4}}$. Additionally, we remark that the universal formula given in~\eqref{QP_Control_Law_relaxed} can be rewritten in a more compact form as a combination of $\mathbf{m}_{\mathrm{Uni-CLF}}^{\star}(\mathbf{x})$ and $\mathbf{n}_{\mathrm{Uni-CLF}}^{\star}(\mathbf{x})$, i.e., $\mathbf{u}_{\mathrm{Uni}}^{\star}(\mathbf{x})=\epsilon_{1}^{*}\mathbf{m}_{\mathrm{Uni-CLF}}^{\star}(\mathbf{x})+\epsilon_{2}^{*}\mathbf{n}_{\mathrm{Uni-CLF}}^{\star}(\mathbf{x})$, where 
\begin{equation*}\label{Parameters}
    \begin{split}
     \left\{\begin{array}{ll}
      \epsilon_{1}^{*}=1,\,\, \epsilon_{2}^{*}=0,&\quad\mathbf{x}\in\mathcal{P}_{1},\\
\epsilon_{1}^{*}=0, \,\,\epsilon_{2}^{*}=1,&\quad\mathbf{x}\in\mathcal{P}_{2},\\
\epsilon_{1}^{*}=\epsilon_{1}, \epsilon_{2}^{*}=\epsilon_{2},&\quad\mathbf{x}\in\mathcal{P}_{3},\\
\epsilon_{1}^{*}=0, \,\,\epsilon_{2}^{*}=0,&\quad\mathbf{x}\in\mathcal{P}_{4}.
\end{array}\right.  
    \end{split}
\end{equation*}
To provide a combination of $\mathbf{m}_{\mathrm{Uni-CLF}}^{\star}(\mathbf{x})$ and $\mathbf{n}_{\mathrm{Uni-CLF}}^{\star}(\mathbf{x})$ that yields a universal formula, an alternative method is proposed in~\cite{Smooth_universal}. The universal formula given in~\cite{Smooth_universal} involves employing a smooth function to define and determine the parameters $\epsilon_{1}^{*}$ and $\epsilon_{2}^{*}$. Generally, different universal formulas can be constructed by carefully selecting the parameters $\epsilon_{1}^{*}$ and $\epsilon_{2}^{*}$, which will result in different state and control behaviours.
\end{Rmk}
\subsubsection{Imcompatible Safe Stabilization}
In practice, the compatibility assumption, i.e., the condition \eqref{Compatibility_conditions},  may not always hold true for certain CLFs and CBFs. In such cases, we need to consider relaxation strategies. Unlike the method described in~\cite{universal_formula} which introduces a slack variable, we choose the following universal formula. 
\begin{equation}\label{QP_Sequential}
    \begin{split}
      \tilde{\mathbf{u}}_{\mathrm{Stg}}^{\star}(\mathbf{x})=\left\{\begin{array}{ll}
\mathbf{m}_{\mathrm{Stg-CLF}}^{\star}(\mathbf{x}),&\quad\mathbf{x}\in\mathcal{S}_{1},\\
\tilde{\mathbf{p}}_{\mathrm{Stg}}^{\star}(\mathbf{x}),&\quad\mathbf{x}\notin\mathcal{S}_{1},
\end{array}\right.  
    \end{split}
\end{equation}
where $\tilde{\mathbf{p}}_{\mathrm{Stg}}^{\star}(\mathbf{x})=\mathbf{m}_{\mathrm{Stg-CLF}}^{\star}(\mathbf{x})-\frac{\Delta(\mathbf{x})}{\mathbf{d}(\mathbf{x})\mathbf{d}(\mathbf{x})^{\top}}\mathbf{d}(\mathbf{x})$, $\Delta(\mathbf{x})=c(\mathbf{x})+\mathbf{d}(\mathbf{x})\mathbf{m}_{\mathrm{Stg-CLF}}^{\star}(\mathbf{x})$, and $\mathcal{S}_{1}=\{\mathbf{x}|\Delta(\mathbf{x})\geq 0\}$.
\begin{Rmk}\label{Imcomp_Interp}
    The universal formula given by~\eqref{QP_Sequential} is the solution of
     \begin{equation}\label{QP_sequential_formu}
        \begin{split}
            \min\limits_{\mathbf{u}\in\mathbb{R}^{m}}&\frac{1}{2}\|\mathbf{u}-\mathbf{m}_{\mathrm{Stg-CLF}}(\mathbf{x})\|^{2}\\
            &c(\mathbf{x})+\mathbf{d}(\mathbf{x}) \mathbf{u}\geq \rho\Gamma(\mathbf{x}).
        \end{split}
    \end{equation}
For the optimization~\eqref{QP_sequential_formu}, we begin by developing a universal stabilization formula, i.e., $\mathbf{m}_{\mathrm{Stg-CLF}}(\mathbf{x})$. When the system violates the safety constraints ($\mathbf{x}\notin\mathcal{S}_{1}$), our objective is to modify the control law $\mathbf{m}_{\mathrm{Stg-CLF}}(\mathbf{x})$ accordingly to ensure safety guarantees. Once the safety constraints are met, i.e., $\mathbf{x}\in\mathcal{S}_{1}$, the control input guarantees the closed-loop stability and safety simultaneously. However, due to that the control input $\mathbf{m}_{\mathrm{Stg-CLF}}(\mathbf{x})$ will be modified if the safety constraints are not met, there is no guarantee for stability.
\end{Rmk}
\section{Universal Formula for Quadrotors Stabilization with Safety Guarantees}\label{Main_results}
In this section, we employ the universal formula presented in~\eqref{QP_Control_Law_relaxed} and~\eqref{QP_Sequential} to address the challenge of safely stabilizing quadrotors in both compatible and incompatible scenarios. Under the position-attitude controller framework, we introduce a universal formula for the position controller to generate setpoint commands for an attitude controller.
\subsection{Quadrotor Dynamics}
For a quadrotor system, two right-handed coordinate frames are used as shown in Fig.~\ref{Coordinates}: they are inertial-frame $\mathcal{F}_{\mathcal{W}}:\{x_{\mathcal{W}},y_{\mathcal{W}},z_{\mathcal{W}}\}$ with $z_{\mathcal{W}}$ pointing upward opposite to the gravity, and the body frame $\mathcal{F}_{\mathcal{B}}:\{x_{\mathcal{B}},y_{\mathcal{B}},z_{\mathcal{B}}\}$ with $x_{\mathcal{B}}$ pointing forward and $z_{\mathcal{B}}$ aligned with the collective thrust direction. The motion of the quadrotor is described by the following equations~\cite{Quadrotor_model}: 
\begin{equation}\label{Quadrotor_model}
\begin{split}
\dot{\mathbf{p}} &=\mathbf{v},\qquad\qquad\quad\,\,\,\, m \dot{\mathbf{v}} =-m g \mathbf{e}_{3}+f \mathbf{R} \mathbf{e}_{3}, \\
\dot{\mathbf{R}}&=\mathbf{R}\bm{\Omega}^{\times},\qquad\qquad\,\,
\mathbf{J}\dot{\bm{\Omega}}=-\bm{\Omega}\times\mathbf{J}\bm{\Omega}+\bm{\tau},
\end{split}
\end{equation}
where the position and velocity of the quadrotor are denoted by $\mathbf{p}=[x,y,z]^{\top}\in\mathbb{R}^{3}$ and $\mathbf{v}=[v_x,v_y,v_z]^{\top}\in\mathbb{R}^{3}$, respectively, and $\mathbf{e}_{3}$ is a unit vector $\mathbf{e}_{3} = [0, 0, 1]^{\top}$. The rotation matrix $\mathbf{R}$ (from $\mathcal{F}_{\mathcal{B}}$ to $\mathcal{F}_{\mathcal{W}}$) for a quadrotor is obtained by three successive rotations: first a rotation of an angle $\psi$ around the $z$ axis, then a rotation of an angle $\theta$ around $y$ axis and finally a rotation an angle $\phi$ around $x$ axis. The resulting rotation matrix is defined as:
\begin{equation*}
\mathbf{R}=\left(\begin{array}{ccc}
c_\phi c_\psi & s_\psi c_\phi & -s_\theta \\
c_\psi s_\theta s_\phi-s_\psi c_\phi & s_\psi s_\theta s_\phi+c_\phi c_\psi & c_\theta s_\phi \\
c_\psi s_\theta c_\phi+s_\phi s_\psi & s_\psi s_\theta c_\phi-c_\psi s_\phi & c_\phi c_\theta
\end{array}\right)
\end{equation*}
where $c$ and $s$ are shorthand forms for cosine and sine, respectively.  The angular variable is defined as $\bm{\zeta}=[\phi, \theta,\psi]^{\top}$, and the angular velocity is denoted by $\bm{\Omega}=[p,q,r]^{\top}$, and $\bm{\Omega}^{\times}$ is a skew-symmetric matrix. For any vector $\mathbf{s}\in\mathbb{R}^{3}$, $\bm{\Omega}^{\times}\mathbf{s}=\bm{\Omega}\times\mathbf{s}$, where $\times$ denotes the vector cross product.  We thereby define the state $\mathbf{x}: =[\mathbf{p}^{\top},\mathbf{v}^{\top},\bm{\zeta}^{\top},\bm{\Omega}^{\top}]^{\top}\in\mathbb{R}^{12}$ and control input $\mathbf{u}: =[f,\bm{\tau}]^{\top}\in\mathbb{R}^{4}$. The mass of the quadrotor is denoted by $m$. The unit vector $\mathbf{e}_{3}=[0,0,1]^{\top}$ specifies the direction of the gravitational force $g$ in the inertial frame $\mathcal{F}_{\mathcal{W}}$. The quadrotor's total force, represented by $f\in\mathbb{R}^{+}$, is the combined effect of the individual thrust forces produced by its rotors. The total torque is denoted as $\bm{\tau}=[\tau_{x},\tau_{y},\tau_{z}]^{\top}$. The quadrotor's inertia matrix, represented by $\mathbf{J}=\mathrm{diag}\left([\mathrm{J}_{x}, \mathrm{J}_{y}, \mathrm{J}_{z}]\right)$, describes how the mass is distributed throughout the quadrotor and determines how it responds to external forces and torques.

\begin{figure}[tp]
 \centering
    \makebox[0pt]{%
    \includegraphics[width=3.2in]{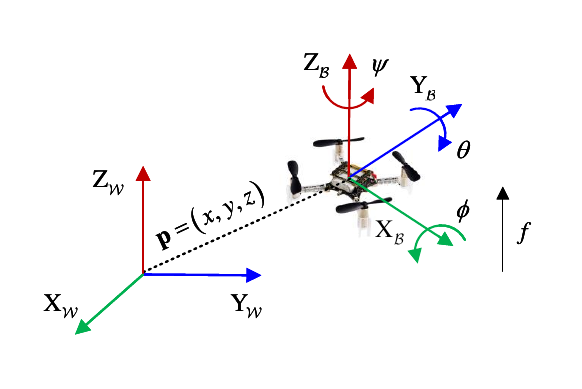}}
    \caption{Quadrotor coordinate frames: System conversion with $\mathcal{W}$ and $\mathcal{B}$ denoting world and body frames, respectively.}
    \label{Coordinates}
\end{figure}
\subsection{An Approximation of Multiple Barrier Functions}\label{Approximation_technique}
The main reason for employing the approximation strategy is that the universal formulas ~\eqref{QP_Control_Law_relaxed} and~\eqref{QP_Sequential} are designed to handle only one single safety constraint. 
When dealing with multiple barrier functions, one possible solution is to first construct a universal control law with a single CBF constraint and then set it as the nominal control input. Next, by repeatedly employing~\eqref{QP_Sequential}, a universal control law for multiple barrier functions can be ultimately obtained. However, since the nominal control input is modified at each time we apply the universal formula, the safety is no longer ensured. To this end, we employ the following approximation strategy, which offers a sufficient condition for multiple barrier function satisfactions. This approximation often leads to a conservative but strict safety guarantees.

Firstly, we define
\begin{equation}\label{CBFs_Approx}
    h(\mathbf{x}):=-\frac{1}{\eta}\mathrm{ln}\left(\sum\nolimits_{l=1}^{p}\exp\left(-\eta h_{l}(\mathbf{x})\right)\right),
\end{equation}
where $l\in\{1,\cdots,p\}$, $\eta>0$. Given that 
\begin{equation*}
    \lim_{\eta\rightarrow\infty}-\frac{1}{\eta}\mathrm{ln}\left(\sum\nolimits_{l=1}^{p}\exp\left(-\eta h_{l}(\mathbf{x})\right)\right)=\min_{l\in\{1,\cdots,p\}}(h_{l}(\mathbf{x})),
\end{equation*}
we can infer that $h(\mathbf{x})=-\frac{1}{\eta}\mathrm{ln}\left(\sum\nolimits_{l=1}^{p}\exp\left(-\eta h_{l}(\mathbf{x})\right)\right)\leq\min_{l\in\{1,\cdots,p\}}(h_{l}(\mathbf{x}))$. Therefore,  $h(\mathbf{x})>0$ is a sufficient condition to ensure $h_{l}(\mathbf{x})>0, l=1,\cdots,p$, which means that safety constraints are satisfied as a consequence. 

In Fig.~\ref{Approx_CBFs}, we present a demonstration of the approximation's performance. Within this visual representation, we establish three barrier functions as follows: $h_{l}(\mathbf{x})=(x-o_{x,l})^{2}+(y-o_{y,l})^{2}-r_{o,l}^{2}, l=1,2,3$, where the centers and radius of the obstacles are $[o_{x,l},o_{y,l},r_{o,l}]=[0.6775,-1.8844,0.3]$, $[o_{x,l},o_{y,l},r_{o,l}]=[0.28182,1.9802,0.3]$, and $[o_{x,l},o_{y,l},r_{o,l}]=[-1.7593,-0.95163,0.3]$. As we see, as the parameter $\eta$ increases, the safe regions originally defined by the three barrier functions can be effectively approximated by one single barrier function.

For the convenience of later discussions, we compute the time derivatives of $h(\mathbf{x})$ as follows:
\begin{equation}\label{First_Time_Derivative}
    \dot{h}(\mathbf{x})=\frac{\sum\nolimits_{l=1}^{p}\left[\exp\left(-\eta h_{l}(\mathbf{x})\right)\cdot\dot{h}_{l}(\mathbf{x})\right]}{\sum\nolimits_{l=1}^{p}\exp\left(-\eta h_{l}(\mathbf{x})\right)}.
\end{equation}
In the following, we define $\xi(\mathbf{x}):=\sum\nolimits_{l=1}^{p}\exp\left(-\eta h_{l}(\mathbf{x})\right)$ and $\zeta(\mathbf{x}):=\sum\nolimits_{l=1}^{p}\left[\exp\left(-\eta h_{l}(\mathbf{x})\right)\cdot\dot{h}_{l}(\mathbf{x})\right]$. Then we know that
\begin{equation}
    \ddot{h}(\mathbf{x})=\frac{\xi(\mathbf{x})\dot{\zeta}(\mathbf{x})-\dot{\xi}(\mathbf{x})\zeta(\mathbf{x})}{\xi^{2}(\mathbf{x})}.
\end{equation}
We will find that 
\begin{equation}\label{Time_Derivative_CBFs_general}
\begin{split}
    \dot{\zeta}(\mathbf{x})&=\chi(\mathbf{x})+\upsilon(\mathbf{x}),\\
    \dot{\xi}(\mathbf{x})&=-\eta\sum\nolimits_{l=1}^{p}\left[\exp\left(-\eta h_{l}(\mathbf{x})\right)\cdot\dot{h}_{l}(\mathbf{x})\right]=-\eta\zeta.
\end{split}
\end{equation}
where $\chi(\mathbf{x})=-\eta\sum\nolimits_{l=1}^{p}\left[\exp\left(-\eta h_{l}(\mathbf{x})\right)\cdot\dot{h}_{l}^{2}(\mathbf{x})\right]$ and $\upsilon(\mathbf{x})=\sum\nolimits_{l=1}^{p}\left[\exp\left(-\eta h_{l}(\mathbf{x})\right)\cdot\ddot{h}_{l}(\mathbf{x})\right]$.
\begin{figure}[tp]
 \centering
    \makebox[0pt]{%
    \includegraphics[width=3in]{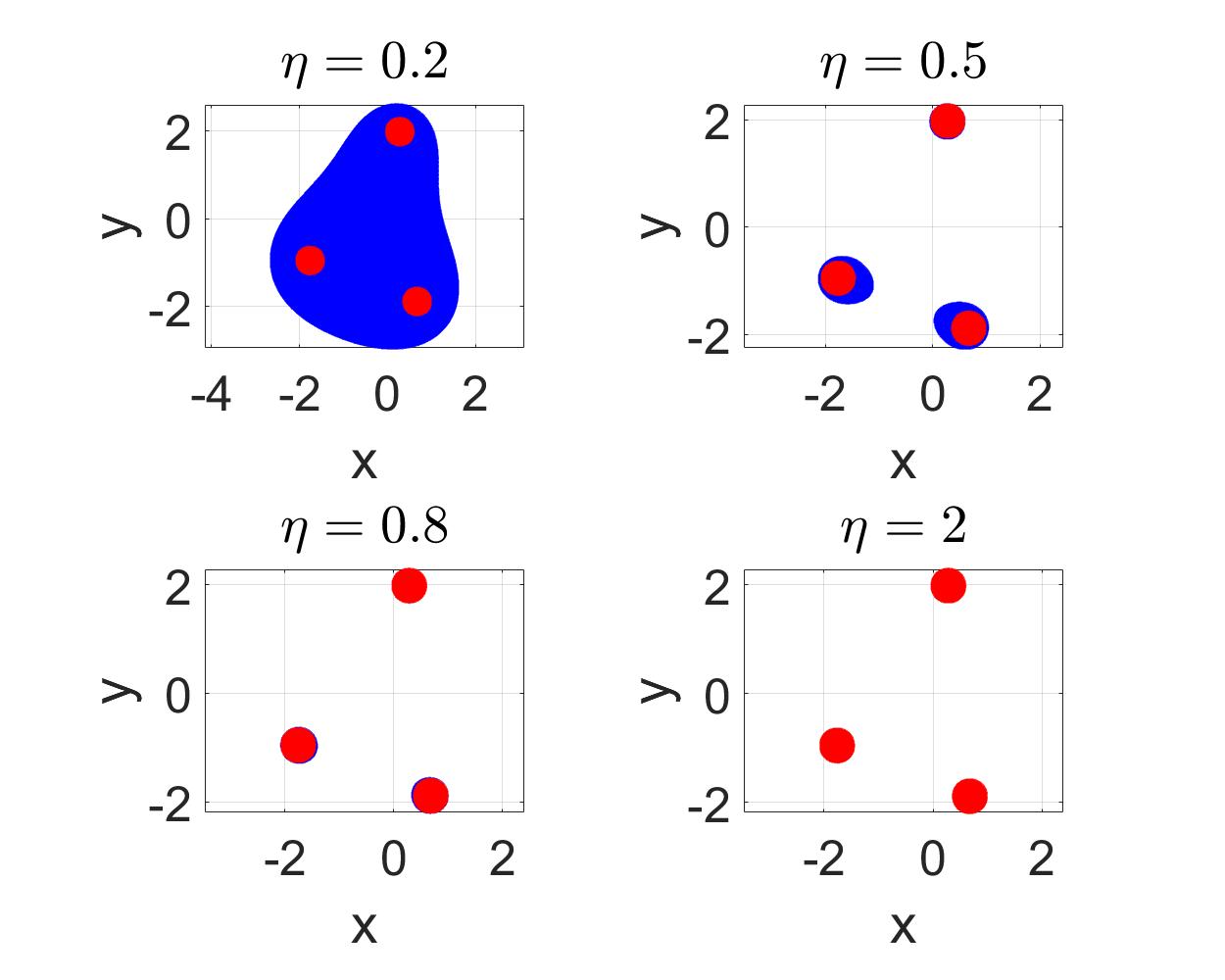}}
    \caption{An approximation of multiple barrier functions: The areas shaded in red represent the unsafe regions defined by the three barrier functions, while the blue regions correspond to the approximated barrier function.}
    \label{Approx_CBFs}
\end{figure}
\subsection{Position-Attitude Controller Architecture}\label{Structure}
The proposed position-attitude controller architecture ~\cite{Geo_control_1,Geo_control_2} involves two controllers: the \textit{position controller} and the \textit{attitude controller}. Fig.~\ref{Control_schematic} shows the control scheme of the quadrotor. Our objective is to generate setpoint commands with stability and safety guarantees within the position controller using the universal formula. Afterward, a PID controller is implemented in the attitude controller to track the trajectory defined by the setpoints.
\subsubsection{Position Controller}
As dipicted in Fig.~\ref{Control_schematic}, the position controller can be divided into i) an altitude controller whose output is the thrust $f$ required to maintain a certain position in $z$; and ii) a lateral $\mathrm{X-Y}$ position controller whose outputs are the required roll and pitch angles, i.e., $\phi_{d}$ and $\theta_{d}$, that will be regulated by the attitude controller. 
\begin{figure*}[tp]
 \centering
    \makebox[0pt]{%
    \includegraphics[width=7in]{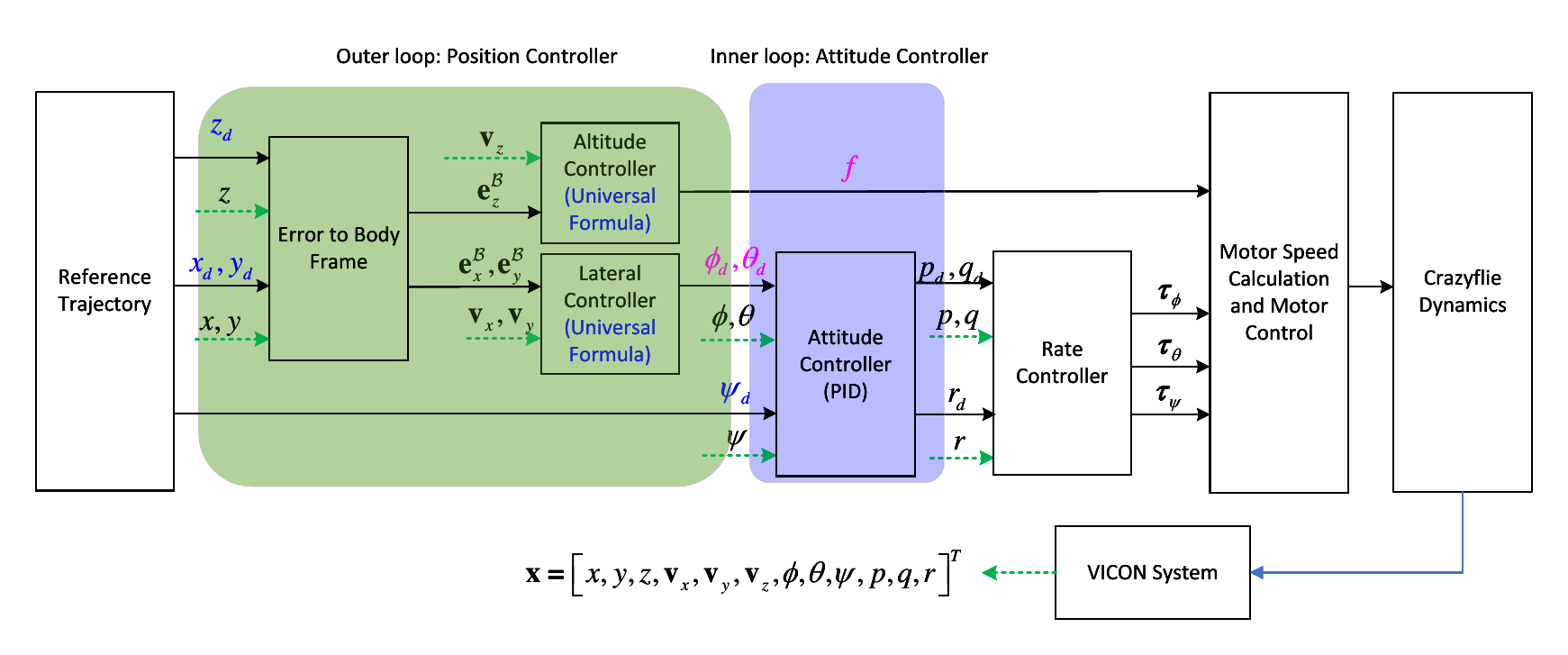}}
    \caption{Control schematic of the quadrotor.}
    \label{Control_schematic}
\end{figure*}

\noindent
$\bullet$ \textbf{Altitude Controller}

We define the altitude errors as follows:
\begin{equation}
    \mathbf{e}_{z}=z-z_{d},\quad \mathbf{e}_{\mathrm{aux},\mathbf{v}_{z}}=\mathbf{v}_{z}-\dot{z}_{d}+ \iota_{1}\mathbf{e}_{z},
\end{equation}
where $\iota_{1}>0$, and the error $\mathbf{e}_{\mathrm{aux},\mathbf{v}_{z}}$ converges to zero when $\mathbf{e}_{z}$ is zero. Next, we construct the following Lyapunov candidate to stabilize the altitude control system.
\begin{equation}\label{Altitude_Lypaunov}
V_{\mathrm{a}}=\frac{1}{2}\mathbf{e}_{\mathrm{a}}^{\top}\mathbf{e}_{\mathrm{a}},
\end{equation} 
where $\mathbf{e}_{\mathrm{a}}=[\mathbf{e}_{z},\mathbf{e}_{\mathrm{aux},\mathbf{v}_{z}}]^{\top}$. Then the time derivative of the Lyapunov candidate is
\begin{equation}
\begin{split}
\dot{V}_{\mathrm{a}}=\mathbf{e}_{z}^{\top}\dot{\mathbf{e}}_{z}+\mathbf{e}_{\mathrm{aux},\mathbf{v}_{z}}^{\top}(\dot{\mathbf{v}}_{z}-\ddot{z}_{d}+\iota_{1}\dot{\mathbf{e}}_{z}).
\end{split}
\end{equation}
According to~\eqref{Quadrotor_model}, we know that $\dot{\mathbf{v}}_{z}=-f/m+g$. Therefore, for the CLF condition, we have 
\begin{equation}
    \begin{split}
        a_{z}(\mathbf{x})&=\mathbf{e}_{z}^{\top}\dot{\mathbf{e}}_{z}+\mathbf{e}_{\mathrm{aux},\mathbf{v}_{z}}^{\top}(g-\ddot{z}_{d}+\dot{\mathbf{v}}_{z}+\iota_{1}\dot{\mathbf{e}}_{z})+\lambda V_{\mathrm{a}},\\
        \mathbf{b}_{z}(\mathbf{x})&=-\mathbf{e}_{\mathrm{aux},\mathbf{v}_{z}}^{\top}/m.
    \end{split}
\end{equation}
We now examine the safety objective related to altitude. To enforce constraints on the altitude state $z$, we utilize the following safety barrier region.
\begin{equation}\label{CBFs}
\begin{split}
    h_{\mathrm{a},1}(z)=z-z_{\min}, \quad h_{\mathrm{a},2}(z)=z_{\max}-z,
\end{split}    
\end{equation}
where $z_{\min}$ and $z_{\max}$ are constants that denote the minimal and maximal altitude values, respectively. Subsequently, we apply~\eqref{CBFs_Approx} to approximate the two barrier functions in~\eqref{CBFs}, which gives $h_{\mathrm{a}}(z):=-\frac{1}{\eta}\mathrm{ln}\left(\sum\nolimits_{l=1}^{2}\exp\left(-\eta h_{\mathrm{a},l}(z)\right)\right)$. To obtain $c_{z}(\mathbf{x})$ and $\mathbf{d}_{z}(\mathbf{x})$ given in~\eqref{ECBF}, we can use~\eqref{First_Time_Derivative}-\eqref{Time_Derivative_CBFs_general} to compute the corresponding parameters. But note that
\begin{equation}
    \upsilon_{z}(\mathbf{x})=s_{z}(\mathbf{x})+t_{z}(\mathbf{x})\mathbf{u}_{z},
\end{equation}
where $s_{z}(\mathbf{x})=(\exp(-\eta h_{\mathrm{a},1}(\mathbf{x}))-\exp(-\eta h_{\mathrm{a},2}(\mathbf{x})))/g$ and $t_{z}(\mathbf{x})=-1/m\cdot(\exp(-\eta h_{\mathrm{a},1}(\mathbf{x}))-\exp(-\eta h_{\mathrm{a},2}(\mathbf{x})))$ according to~\eqref{Time_Derivative_CBFs_general} and~\eqref{CBFs}. Then we have
\begin{small}
\begin{equation}\label{CBF_parameters}
\begin{split}
c_{z}(\mathbf{x})&=\frac{\xi_{z}(\mathbf{x})(\chi_{z}(\mathbf{x})+s_{z}(\mathbf{x}))-\dot{\xi}_{z}(\mathbf{x})\zeta_{z}(\mathbf{x})}{\xi^{2}(\mathbf{x})}+\mathbf{K}_{z}^{\top}\mathcal{H}_{z}, \\
\mathbf{d}_{z}(\mathbf{x})&=\frac{t_{z}(\mathbf{x})}{\xi_{z}(\mathbf{x})}.
\end{split}
\end{equation}
\end{small}

\begin{Lem}\label{Safe_stabilization}
    Assume that i) the parameters $a_{z}(\mathbf{x})$, $\mathbf{b}_{z}(\mathbf{x})$, $c_{z}(\mathbf{x})$, and $\mathbf{d}_{z}(\mathbf{x})$ can be successfully obtained, and ii) the CLF defined in~\eqref{Altitude_Lypaunov}, is compatible with the approximated CLF $h_{\mathrm{a}}(z)$ for system~\eqref{Quadrotor_model}. With the use of the universal formula ~\eqref{QP_Control_Law_relaxed}, a safe stabilization with respect to altitude can be achieved. 
\end{Lem}
\begin{proof}
   According to Theorem~\ref{Compatible_Qp}, with the parameters $a_{z}(\mathbf{x})$, $\mathbf{b}_{z}(\mathbf{x})$, $c_{z}(\mathbf{x})$, and $\mathbf{d}_{z}(\mathbf{x}),$ the universal fomula control law~\eqref{QP_Control_Law_relaxed} ensures that $\dot{V}_{\mathrm{a}}=a_{z}(\mathbf{x})+\mathbf{b}_{z}(\mathbf{x})f\leq 0$ and $\dot{h}_{\mathrm{a}}=c_{z}(\mathbf{x})+\mathbf{d}_{z}(\mathbf{x})f\geq 0.$ This is in accordance with Definition~\ref{ECBF_Conditon}, which implies that $h_{\mathrm{a}}(z)\geq 0$. Given that $\min\left(h_{\mathrm{a},1}(z),h_{\mathrm{a},2}(z)\right)\geq h_{\mathrm{a}}(z)\geq 0$, both safety and stability requirements are met. As a result, the safe altitude stabilization is accomplished.
\end{proof}

Note that the assumption in Lemma~\ref{Safe_stabilization} on the compatibility of CLF and CBF for system~\eqref{Quadrotor_model} is not always satisfied, as discussed in~\ref{G_Universal}. To address this problem, we can use the universal formula provided in~\eqref{QP_Sequential}. In this case, we can ensure the safety constraint is rigorously met, but without any guarantees of stability. The proof can be conducted by referring to Remark~\ref{Imcomp_Interp} and subsequently adjusting the conditions for the CLF and CBF in the proof of Lemma~\ref{Safe_stabilization} accordingly.

\noindent
$\bullet$ \textbf{Lateral $\text{X-Y}$ Controller}

Let us define errors associated with the dynamics of the quadrotor as follows:
\begin{equation}\label{X_Y_error_dynamics}
\begin{split}
    &\mathbf{e}_{x}=x-x_{d},\quad \mathbf{e}_{y}=y-y_{d},\\
&\mathbf{e}_{\mathrm{aux},\mathbf{v}_{x}}=\mathbf{v}_{x}-\dot{x}_{d}+\iota_{2}\mathbf{e}_{x},\mathbf{e}_{\mathrm{aux},\mathbf{v}_{y}}=\mathbf{v}_{y}-\dot{y}_{d}+\iota_{3}\mathbf{e}_{y},
\end{split}
\end{equation}
where $\iota_{2}, \iota_{3}>0$ are positive constants. First we define $\mathbf{e}_{\mathrm{p}}=[\mathbf{e}_{x},\mathbf{e}_{y},\mathbf{e}_{\mathrm{aux},\mathbf{v}_{x}},\mathbf{e}_{\mathrm{aux},\mathbf{v}_{y}}]$. Then the Lyapunov candidate $V_{\mathrm{p}}$ is constructed by~\cite{CLF_Construction}:
\begin{equation}\label{xy_lyapunov}
    V_{\mathrm{p}}=\frac{1}{2}\mathbf{e}_{\mathrm{p}}^{\top}\mathbf{e}_{\mathrm{p}}.
\end{equation}
The time derivative of the CLF candidate can be derived by
\begin{equation}
\begin{split}    \dot{V}_{\mathrm{p}}
=a_{xy}(\mathbf{x})+b_{xy}(\mathbf{x})\mathbf{u}_{xy}.
\end{split}    
\end{equation}
where $\mathbf{u}_{xy}=[\mathbf{u}_{x},\mathbf{u}_{y}]^{\top}=
[\phi, \theta]^{\top}$. Based on the dynamical model given in~\eqref{Quadrotor_model} and relations
\begin{equation}
\begin{aligned} \ddot{x}=\frac{\mathbf{u}_{x} s\psi-\mathbf{u}_{y} c\psi}{g}, \quad\ddot{y}=\frac{\mathbf{u}_{x} c\psi+\mathbf{u}_{y} s\psi}{g},
\end{aligned}
\end{equation}
we can obtain that
\begin{small}
\begin{equation*}
\begin{split}
a_{xy}(\mathbf{x})&=\mathbf{e}_{\mathrm{p}}^{\top}\cdot\left[\dot{\mathbf{e}}_{x}, \dot{\mathbf{e}}_{y}, -\ddot{x}_{d}+\iota_{2}(\mathbf{v}_{x}-\dot{\mathbf{x}}_{d}), -\ddot{y}_{d}+\iota_{3}(\mathbf{v}_{y}-\dot{\mathbf{y}}_{d})\right], \\
\mathbf{b}_{xy}(\mathbf{x})&=[\mathbf{e}_{\mathrm{aux},\mathbf{v}_{x}}\cdot s\psi/g+\mathbf{e}_{\mathrm{aux},\mathbf{v}_{y}}\cdot c\psi/g,\\
&\quad-\mathbf{e}_{\mathrm{aux},\mathbf{v}_{x}}\cdot c\psi/g+\mathbf{e}_{\mathrm{aux},\mathbf{v}_{y}}\cdot s\psi/g].
\end{split}
\end{equation*}
\end{small}

Next, suppose that the safety sets are defined by the following barrier functions: 
\begin{equation}\label{constructed_CBFs}
    h_{l}(x,y)=(x-o_{l,x})^{2}+(y-o_{l,y})^{2}-r_{o,l}^{2}\geq 0,
\end{equation}
where $l\in\{1,\cdots,p\}$. Then the time derivatives of $h_{l}(x,y)$ are given by
\begin{equation}\label{time_derivative_CBF}
\begin{split}
    \dot{h}_{l}(x,y)&=2(x-o_{l,x})\dot{x}+2(y-o_{l,y})\dot{y},\\
    \ddot{h}_{l}(x,y)&=2\dot{x}^{2}+2(x-o_{l,x})\ddot{x}+2\dot{y}^{2}+2(y-o_{l,y})\ddot{y}.
\end{split}
\end{equation}
We employ the smooth approximation technique in Section~~\ref{Approximation_technique} to handle multiple barrier functions. Similar to the altitude case, the parameters for $c_{xy}(\mathbf{x})$ and $\mathbf{d}_{xy}(\mathbf{x})$ can be obtained with the same manner as~\eqref{CBF_parameters}. The only unclear parameters $s_{xy}(\mathbf{x})$ and $t_{xy}(\mathbf{x})$ are derived as follows:
\begin{equation*}
\begin{split}
&s_{xy}(\mathbf{x})=\sum\nolimits_{l=1}^{p}\left[\exp\left(-\eta h_{l}(\mathbf{x})\right)\cdot(2\dot{x}^{2}+2\dot{y}^{2}))\right],\\
&t_{xy}(\mathbf{x})=\sum\nolimits_{l=1}^{p}2\exp\left(-\eta h_{l}(\mathbf{x})\right)/g\\
&\cdot[(x-o_{l,x})s\psi+(y-o_{l,y})c\psi,(x-o_{l,x})c\psi+(y-o_{l,y})s\psi].
\end{split}
\end{equation*}
Next, we also utilize~\eqref{CBF_parameters} to compute both $c_{xy}(\mathbf{x})$ and $d_{xy}(\mathbf{x})$. With these calculations in place, we are now prepared to apply the universal formula ~\eqref{QP_Control_Law_relaxed} or~\eqref{QP_Sequential} to determine $\mathbf{u}_{xy}$.
\begin{Lem}\label{Safe_stabilization_XY}
     Suppose that i) the parameters $a_{xy}(\mathbf{x})$, $\mathbf{b}_{xy}(\mathbf{x})$, $c_{xy}(\mathbf{x})$ , and $\mathbf{d}_{xy}(\mathbf{x})$, are available, and ii) assume the CLF in~\eqref{xy_lyapunov} is compatible with the approximated CBF $h(x,y)$ for system~\eqref{Quadrotor_model}, then the universal formula in~\eqref{QP_Control_Law_relaxed} ensures a safe $X-Y$ position stabilization.
\end{Lem}
\begin{proof}
This lemma can be proved by following the same steps as in Lemma~\ref{Safe_stabilization}.
\end{proof}
Furthermore, if the compatibility condition stated in Lemma~\ref{Safe_stabilization_XY} is not met, we will employ the universal formula as detailed in~\ref{QP_Sequential}. This solution is akin to the altitude controller and will yield the same conclusion.
\subsubsection{Attitude Controller}
Upon receiving commands for the desired Euler angles, namely $\phi_{d}$, $\theta_{d}$, and $\psi_{d}$, the role of the attitude controller is to quickly stabilize these angles. Typically, the attitude controller operates at a significantly higher frequency compared to the execution of the position controller. Furthermore, the attitude controller can also be regarded as a regulator for the rate controller, determining the suitable reference values for angular velocity to achieve stability at a specific desired angular orientation for the quadrotor. In this paper, we utilize a PID controller to attain attitude control. Notably, the geometric controller outlined in~\cite{Geometric_Controller} is also a viable choice.
\section{Universal Formula with Disturbances and Input Constraints}\label{Universal_with_Disturbance}
In this section, we first introduce the concept of ISS~\cite{ISS} and ISSf~\cite{ISSf-High-Order} to account for input disturbances in the system dynamics. Next, we address the issue of potential control input violations when utilizing the universal formula. To mitigate this concern, we develop a projection-based strategy, which effectively maps infeasible universal control inputs into the admissible range for control inputs.
\subsection{Universal Formula with Disturbance Rejection}
We regard~\eqref{Affine_Control_System} as our nominal dynamic model and introduce perturbations to its dynamics with the disturbance $\bm{\omega}: \mathbb{R}^{n}\rightarrow\mathbb{R}^{n}$. Subsequently, we investigate the following dynamics:
\begin{equation}\label{Affine_Control_Disturbed}
	    \dot{\mathbf{x}}=\mathbf{f}(\mathbf{x})+\mathbf{g}(\mathbf{x})\mathbf{u}+\bm{\omega}(\mathbf{x}),
\end{equation}
where $\|\bm{\omega}\|_{\infty}\leq \bar{\omega}$, and $\bar{\omega}$ is a constant.
\begin{Def}\label{ISS_Def}
    (Input-to-state stable, (ISS)~\cite{ISS}) The system~\eqref{Affine_Control_Disturbed} is ISS if there exists $\vartheta\in\mathcal{KL}$ and $\iota\in\mathcal{K}_{\infty}$ such that
    \begin{equation}
        \|\mathbf{x}(t)\|\leq\vartheta(\|\mathbf{x}(0)\|,t)+\iota(\|\bm{\omega}\|_{\infty}).
    \end{equation}
\end{Def}
\begin{Lem}\label{ISS_Proposition}
(\cite{taylor2023robust}) Let $V: \mathbb{R}^{n}\rightarrow\mathbb{R}_{+}$ be a continuously differentiable function. If there exist functions $k_{1}, k_{2}, k_{3},\varepsilon\in\mathbb{R}_{>0}$ such that 
\begin{equation}\label{ISS_Condition}
    \begin{split}
    &\qquad\quad k_{1}\|\mathbf{x}\|^{2}\leq V(\mathbf{x})\leq k_{2}\|\mathbf{x}\|^{2},\\
        &a(\mathbf{x})+\mathbf{b}(\mathbf{x})\mathbf{u}+1/\varepsilon\left\|\frac{\partial V}{\partial\mathbf{x}}(\mathbf{x})\right\|^{2}\leq -k_{3}\|\mathbf{x}\|^{2},
    \end{split}
\end{equation}
then the system \eqref{Affine_Control_System} is ISS.
\end{Lem}
To introduce the concept of ISSf, we begin by defining the following sets.
\begin{equation}\label{Invariant_Set_Disturbed}
		\begin{aligned}
			\mathcal{C}_{\bm{\omega}} & \triangleq\left\{\mathbf{x} \in \mathbb{R}^{n}: h(\mathbf{x})+\gamma(\|\bm{\omega}\|_{\infty}) \geq 0\right\} \\
			\partial \mathcal{C}_{\bm{\omega}} & \triangleq\left\{\mathbf{x} \in \mathbb{R}^{n}: h(\mathbf{x})+\gamma(\|\bm{\omega}\|_{\infty})=0\right\} \\
			\operatorname{Int}(\mathcal{C}_{\bm{\omega}}) & \triangleq\left\{\mathbf{x} \in \mathbb{R}^{n}: h(\mathbf{x})+\gamma(\|\bm{\omega}\|_{\infty})>0\right\},
		\end{aligned}
\end{equation}
where $\gamma\in\mathcal{K}_{\infty}$, $\|\bm{\omega}\|_{\infty}\leq\overline{\omega}$, $\overline{\omega}>0$ is a constant.  

Similar to Definition~\ref{ECBF_Conditon}, we consider the state constraints with high relative degrees. In this case, we need to extend the concept of ECBF by accounting for the model uncertainties $\bm{\omega}(\mathbf{x})$. In this regard, we recall the following lemma.
\begin{Lem}\label{ISSf_Theory}
(\cite{taylor2023robust},~\cite{ISSf}) Consider $\mathcal{C}_{\bm{\omega}}$ be defined by \eqref{Invariant_Set_Disturbed} and $h$ is an $r$th continuously differentiable function. If $\mathbf{x}_{0}\in\mathcal{C}_{\bm{\omega}}$ and there exists a controller $\mathbf{u}:\mathbb{R}^{n}\rightarrow\mathbb{R}^{m}$ such that 
\begin{equation}\label{HOCBF_ISSf}
    \begin{split}
        &c(\mathbf{x})+\mathbf{d}(\mathbf{x})\mathbf{u}-\frac{1}{\xi}\left\|\frac{\partial^{(r)} h}{\partial\mathbf{x}^{(r)}}(\mathbf{x})\right\|\geq 0,
    \end{split}
\end{equation}
where $r$ indicates that $h$ is a $r$th continuously differentiable function, $\xi>0$ is a constant. Then the system \eqref{Affine_Control_System} is ISSf, wherein the set $\mathcal{C}_{\bm{\omega}}$ is forward invariant.
\end{Lem}

We now apply Lemma~\ref{ISS_Proposition} and  Lemma~\ref{ISSf_Theory} to address the safe stabilization problem of a quadrotor, which takes the impact of input disturbances into account. 

As a consequence of thrust generation and vehicle motion, the quadrotor dynamics may include body drag and parasitic aerodynamic forces~\cite{Disturbance_Model}. In this case, we modify the dynamics~\eqref{Quadrotor_model} to
\begin{equation}\label{Quadrotor_model_disturb}
\begin{split}
\dot{\mathbf{p}} &=\mathbf{v},\qquad\qquad\quad\,\,\,\, m \dot{\mathbf{v}} =m g \mathbf{e}_{3}-f \mathbf{R} \mathbf{e}_{3}+\mathbf{R}f_{\mathrm{a}}, \\
\dot{\mathbf{R}}&=\mathbf{R}\bm{\Omega}^{\times}\qquad\qquad\,\,
\mathbf{J}\dot{\bm{\Omega}}=-\bm{\Omega}\times\mathbf{J}\bm{\Omega}+\bm{\tau},
\end{split}
\end{equation}
where $f_{\mathrm{a}}$ denotes the body drag and parasitic aerodynamic forces due to the trust generation and vehicle motion for the quadrotor. In this case, $\mathbf{R}f_{\mathrm{a}}$ is recognized as the disturbance term $\bm{\omega}(\mathbf{x})$ as in~\eqref{Affine_Control_Disturbed}. 

For system~\eqref{Quadrotor_model_disturb}, we can effectively follow the procedure outlined in Section~\ref{Structure} for constructing a cascaded control architecture. Note that we will continue to utilize the Lyapunov functions and barrier functions presented in~\eqref{Altitude_Lypaunov},~\eqref{CBFs},~\eqref{xy_lyapunov}, and~\eqref{constructed_CBFs}. The only distinction lies in the stability and safety condition, where we employ the ISS-CLF and ISSf-CBF conditions as provided in~\eqref{ISS_Condition} and~\eqref{HOCBF_ISSf}. As claimed in Lemma~\ref{ISSf_Theory} and outlined in Definition~\ref{ISS_Def}, we can ensure both the forward invariance of a smaller safety set $\mathcal{C}_{\bm{\omega}}$ and the convergence of a neighboring region around the equilibrium point. One can follow the proof in~\cite{taylor2023robust} and~\cite{ISSf} to obtain the conclusions.
\begin{figure}[tp]
 \centering
    \makebox[0pt]{%
    \includegraphics[width=2.2in]{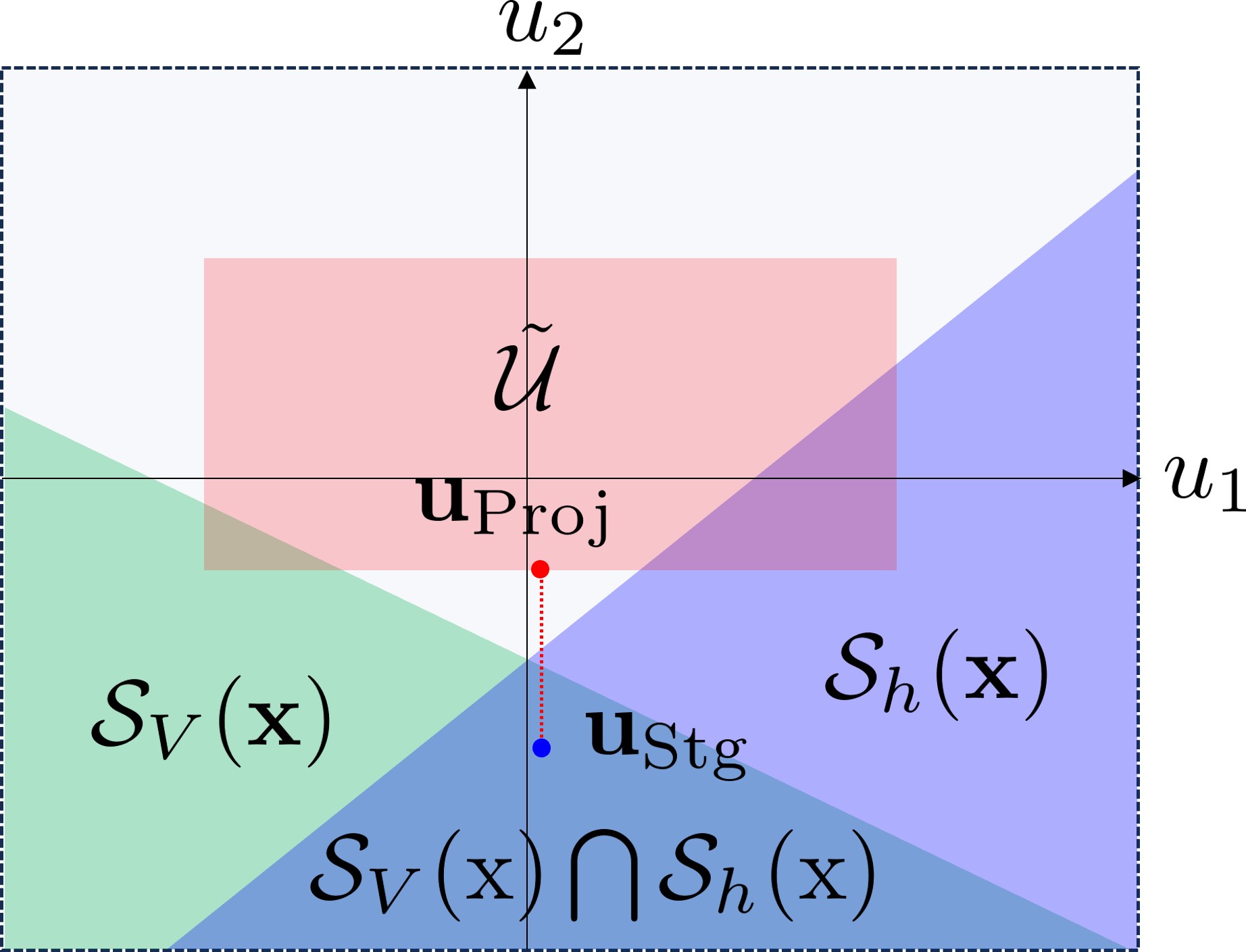}}
    \caption{A graphical interpretation of the projection idea.}
    \label{Projection_Inter}
\end{figure}
\subsection{Universal Formula with Input Constraints}
To account for the limited capability of the attitude controller in stabilizing a specific range of angles and the constraints on quadrotor thrust, we need to establish constraints for both $\phi$ and $\theta$ control inputs, as well as bounds on thrust $f$. Consequently, the attainable control input vector $\tilde{\mathbf{u}}=[\phi,\theta,f]^{\top}$ generated by the vehicle will be confined within the admissible set $\tilde{\mathcal{U}}$.
\begin{equation}
     \tilde{\mathcal{U}}=\{\tilde{\mathbf{u}}\in\mathbb{R}^{m}|\tilde{\mathbf{u}}_{\mathrm{min}}\leq \tilde{\mathbf{u}}\leq \tilde{\mathbf{u}}_{\mathrm{max}}\},
\end{equation}
where $\tilde{\mathbf{u}}_{\mathrm{min}}$ and $\tilde{\mathbf{u}}_{\mathrm{max}}$ are constant vectors, which denote the minimal and maximal values of the control input, respectively.  

To ensure that the universal formula derived in Section~\ref{Main_results} and Section~\ref{Universal_with_Disturbance}, denoted as $\mathbf{u}_{\mathrm{Stg}}$, remains within the bounds of $\tilde{\mathcal{U}}$, we perform a projection operation that aligns it with the nearest point within the set $\tilde{\mathcal{U}}$. To clarify this concept, we introduce the following definitions: $\mathcal{S}_{V}(\mathbf{x})=\{\mathbf{x}\in\mathbb{R}^{n}|a(\mathbf{x})+\mathbf{b}(\mathbf{x})\mathbf{u}\leq -\kappa\sigma(\mathbf{x})\}$ and $\mathcal{S}_{h}(\mathbf{x})=\{\mathbf{x}\in\mathbb{R}^{n}|c(\mathbf{x})+\mathbf{d}(\mathbf{x})\mathbf{u}\geq \rho\Gamma(\mathbf{x})\}$ We then provide a graphical interpretation in Fig.~\ref{Projection_Inter}~\footnotetext[1]{To facilitate graphical interpretation, we have chosen to set the control input dimension to 2. However, the graph interpretation can be extended to cases of arbitrary dimensions.}. We denote the obtained control input as $\mathbf{u}_{\mathrm{proj}}$, then the optimal control input can be obtained by solving the following optimization problem.
\begin{equation}\label{QP_projection}
        \begin{split}
            \min\limits_{\mathbf{u}_{\mathrm{proj}}\in\mathbb{R}^{m}}&\frac{1}{2}\|\mathbf{u}_{\mathrm{proj}}-\mathbf{u}_{\mathrm{Stg}}\|^{2}\\
            &\tilde{\mathbf{u}}_{\mathrm{min}}\leq \mathbf{u}_{\mathrm{Proj}}\leq \tilde{\mathbf{u}}_{\mathrm{max}}.
        \end{split}
    \end{equation}
To facilitate on-board implementation and avoid the need for solving the entire optimization problem, we use its analytical solution, which incorporates the use of a saturation function:
\begin{equation}\label{Saturation}
    \mathrm{SAT}(u)=\left\{\begin{array}{l}
u_{\mathrm{max}}, \quad u>u_{\mathrm{max}}, \\
u, \qquad\,\,\, u_{\mathrm{min}} \leq u \leq u_{\mathrm{max}}, \\
u_{\mathrm{min}}, \quad\,\, u<u_{\mathrm{min}}.
\end{array}\right.
\end{equation}
Consequently, applying $\phi_{\mathrm{Stg}}$, $\theta_{\mathrm{Stg}}$, and $f_{\mathrm{Stg}}$ to~\eqref{Saturation} allows us to determine the admissible control input $\mathbf{u}_{\mathrm{proj}}=[\mathrm{SAT}(\phi_{\mathrm{Stg}}),\mathrm{SAT}(\theta_{\mathrm{Stg}}),\mathrm{SAT}(f_{\mathrm{Stg}})]$. 
\begin{Rmk}
   As shown in Fig.~\ref{Projection_Inter}, the control input $\mathbf{u}_{\mathrm{Proj}}$ falls within the admissible range of control inputs, denoted as $\tilde{\mathcal{U}}$, but it lacks stability and safety guarantees, as it does not belong to either $\mathcal{S}_{V}(\mathbf{x})$ or $\mathcal{S}_{h}(\mathbf{x})$. This relaxation of the stability and safety requirements may be acceptable under certain circumstances, particularly when the system stability is not a strict necessity, and there are some safety margins. 
\end{Rmk}
\begin{figure*}[tp]
 \centering
 \hspace*{-2cm} 
    \makebox[0pt]{%
    \includegraphics[width=4.5in]{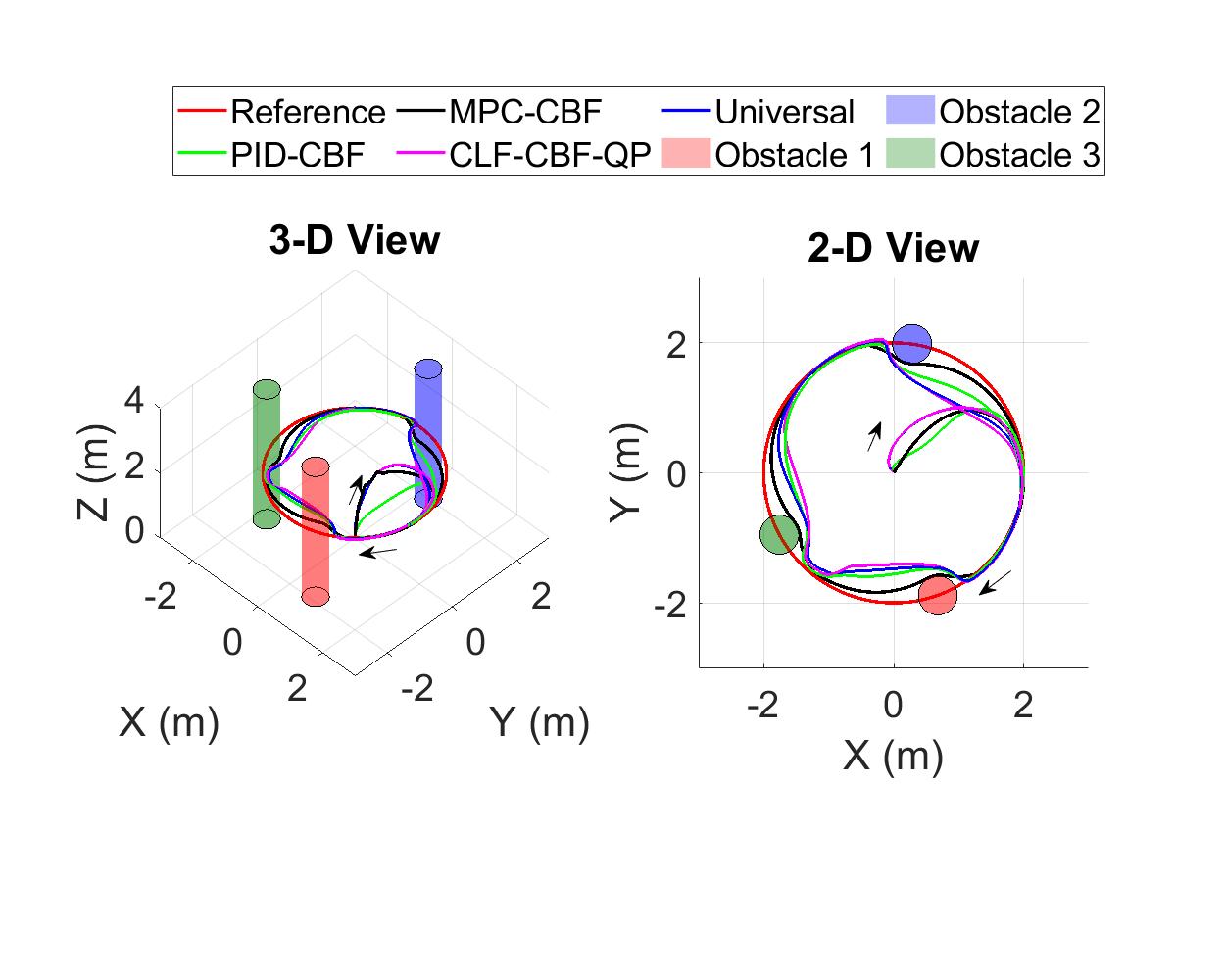}}
    \caption{Simulation - The environment settings and safe tracking performances of four different algorithms.}
    \label{Envi_and_Diff}
\end{figure*}
\begin{figure*}[tp]
 \centering
    \makebox[0pt]{%
    \includegraphics[width=5.5in]{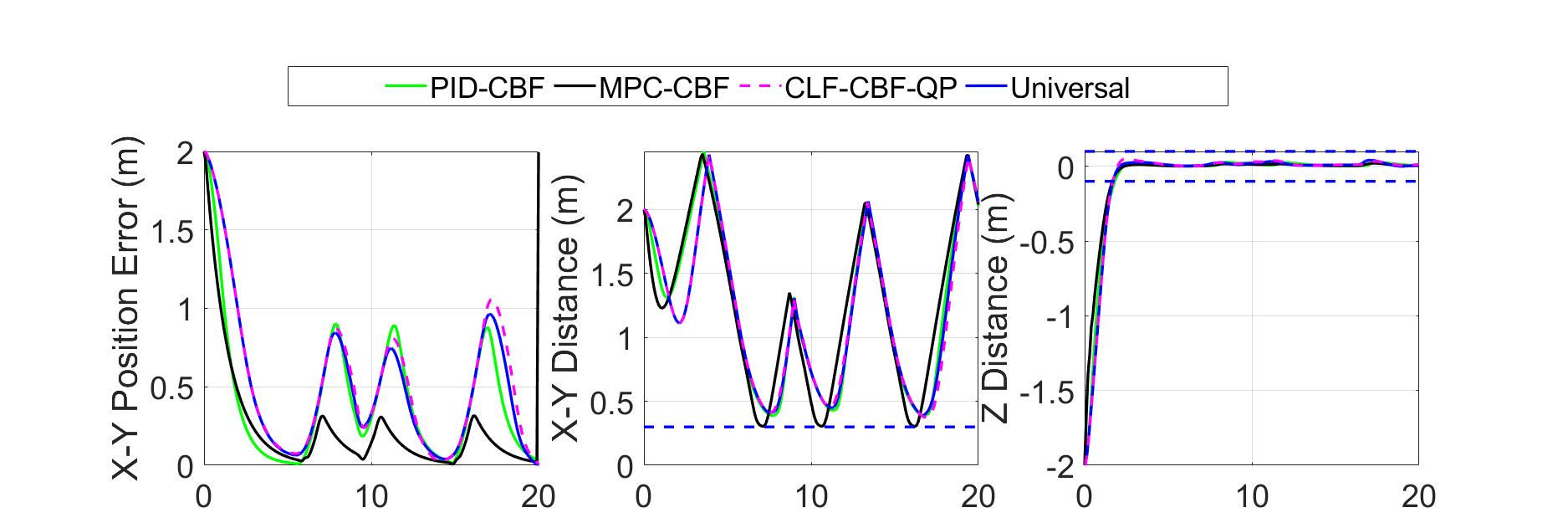}}
    \caption{Simulation - Tracking errors and safety evaluations of four algorithms in ensuring safe tracking in both $\text{X-Y}$ and $\text{Z}$ axes.}
    \label{Performance_Comparison}
\end{figure*}
\begin{figure*}[tp]
 \centering
    \makebox[0pt]{%
    \includegraphics[width=5.5in]{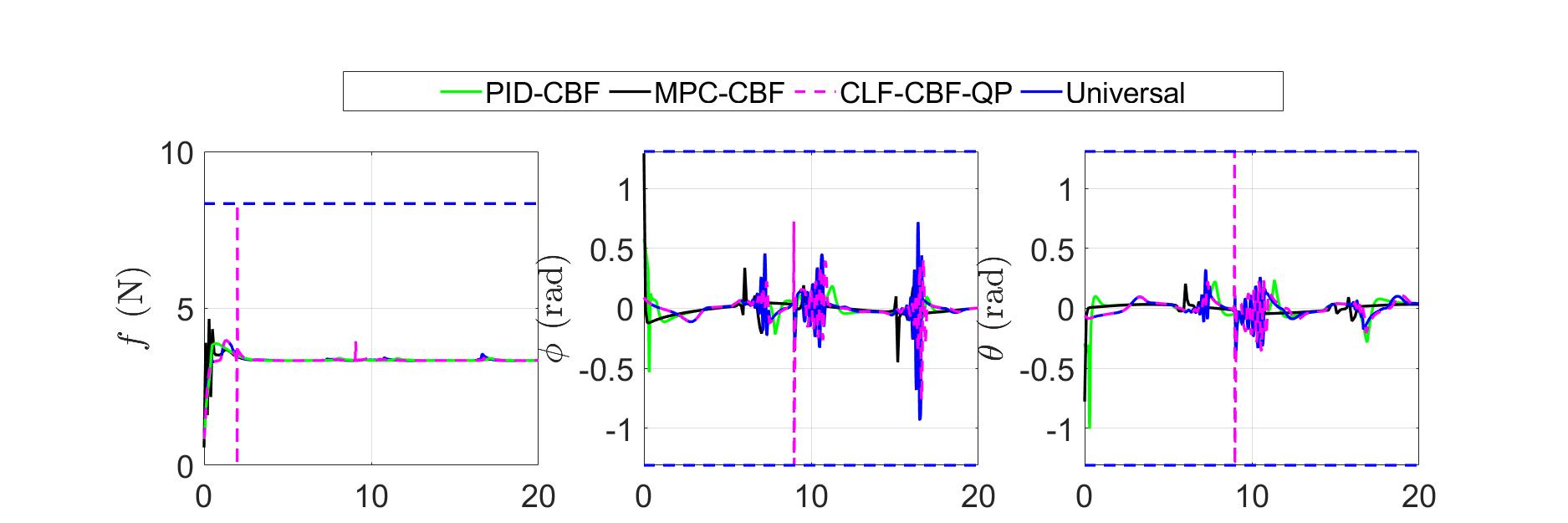}}
    \caption{Simulation - Control input behavior of the four different algorithms. }
    \label{Control_Input_Comparison}
\end{figure*}
\section{Simulation Results and Experimental Tests}
In this section, we present a comprehensive validation of our methodology with both simulation results and real-world experiment tests. Our objective is to demonstrate the effectiveness of our algorithm, which has a fast computation speed and superior performance as compared with the existing methods. To assess its performance, we conduct a comparative analysis against three state-of-the-art approaches, namely, PID-CBF~\cite{quadrotor_safe_stabilization3}, MPC-CBF~\cite{MPC_CBF_Implementation}, and CLF-CBF-QP~\cite{Zeroing_CBF}.
\subsection{Implementation Details}
\begin{enumerate}
    \item The simulated data were processed using Matlab 2019a on a 64-bit Intel Core i7-9750H with a 2.6-GHz processor. For onboard data processing, we employed Crazyflie 2.1, equipped with a 32-bit, 168-MHz ARM microcontroller featuring 1Mb flash memory and a floating-point unit.
    \item For simulations and experiments, we consider a desired reference trajectory $\mathbf{r}_{d}=[1\sin(0.4 t), 1\cos(0.4 t), 0.3]^{\top}$, which is a circle in the $\mathrm{X-Y}$ plane centered at the origin. As shown in Fig.~\ref{Envi_and_Diff}, the drone starts at the initial position $\mathbf{p}_{0}=[0,0,0]^{\top}$, and the control objective is to follow $\mathbf{r}_{d}$ and avoid three cylinder-shaped obstacles of radius $r_{1}=r_{2}=r_3=0.2$ centered at $[-0.9,0.7,0.3]^{\top}$, $[0.9,0.7,0.3]^{\top}$, and $[0.806,-0.806,0.3]^{\top}$, respectively. 
    The height of the cylinder is set to be $z=0.5\mathrm{m}$. The total running time for the trajectory is set to be $t_{\mathrm{total}}=20\mathrm{s}$.   
    \item In Section~\ref{Structure}, we have detailed our choice of Lyapunov functions, by specifically selecting~\eqref{Altitude_Lypaunov} for altitude stabilization and~\eqref{xy_lyapunov} for X-Y position stabilization. Furthermore, We have determined the parameter values for the barrier functions along the $\text{Z}$ direction as follows: For both simulations and experiments, we set $z_{\mathrm{min}}=0.25\mathrm{m}$ and $z_{\mathrm{max}}=0.35\mathrm{m}$, in accordance with the formulation described in equation \eqref{CBFs}.
    \item For the experiment, the real data are collected in the Control Systems Group laboratory of TU/e. The setup of the experiment is shown in Fig.~\ref{Environment_Exp}, which consists of one drone moving in a $3\mathrm{m}\times 3\mathrm{m}\times 3\mathrm{m}$ region with three cylinder-shaped obstacles. We use a VICON motion capture system to track both the quadrotor and obstacles and regard the data (composed of $(x,y,z)$ positions and velocities collected with a data capacity set to $100$ samples) as the true quantity.  
    \item In the simulation, the modeling parameters are set as follows: The mass of the quadrotor, denoted as $m$, is equal to $0.34$, which includes the weight of $4$ markers. The gravitational acceleration, represented as $g$, is approximately $9.81 \mathrm{m/s^{2}}$. The quadrotor's inertia matrix is measured as $\mathrm{diag}\left([\mathrm{J}_{x}, \mathrm{J}_{y}, \mathrm{J}_{z}]\right) = \mathrm{diag}\left([321.86260, 305.82542, 576.25587]\right) \times 10^{-6}\mathrm{m/s^{2}}$.
    \item For real-time communication, as well as subscribing to pose information and publishing setpoints via the Crazyradio PA USB dongle, we rely on the $\mathrm{crazyflie\_ros}$ API. To facilitate communication and low-level attitude control, we leverage the open-source Robot Operating System (ROS) package tailored for Crazyflie.
    \item To demonstrate the superiority of the proposed controller, we design simulations to compare it with three start-of-the-art safety controllers:  PID-CBF~\cite{quadrotor_safe_stabilization3}, MPC-CBF~\cite{MPC_CBF_Implementation}, and CLF-CBF-QP~\cite{Zeroing_CBF}. We record the following parameters that might be of interest to our readers:
\begin{itemize}
    \item Regarding the CLF specified in~\eqref{CLF_Condition}, we select $\lambda=2$. Concerning the CBF, we set the parameter matrix $\mathbf{K}$ to $\mathbf{K}=[4,4]^{\top}$.

    \item PID-CBF: In the PID-CBF approach, we utilize PID controllers for both tracking tasks. The specific parameter values are set as follows: $\mathbf{k}_{x}=[3.0,0.0,3.6]^{\top}$, $\mathbf{k}_{y}=[3.0,0.0,3.6]^{\top}$ and $\mathbf{k}_{z}=[3.0,0.0,3.0]^{\top}$. Then for attitude controller, $\mathbf{k}_{\phi}=[380.0,0.0,40.0]^{\top}$, $\mathbf{k}_{\theta}=[380.0,0.0,40.0]^{\top}$ and $\mathbf{k}_{\psi}=[380.0,0.0,40.0]^{\top}$.

    \item MPC-CBF: In the case of MPC-CBF, we set the prediction horizon to $N=10$, and we employ the solver ``IPOPT''  for optimization. 

    \item CLF-CBF-QP: Within the CLF-CBF-QP framework, the penalty factor in (33) of~\cite{Zeroing_CBF} is selected to be $\delta=50$. We use the solver ``quadprog'' for evaluation.
\end{itemize}

    \item We recommend the readers watch the video summary of the experiments to get a clearer view of the setup and results. It can be found at: https://youtu.be/kOHG0QJd0xg.
\end{enumerate}
\subsection{Simulation Results}
In Fig.\ref{Envi_and_Diff}, we present a snapshot illustrating the safe tracking behaviors achieved through the implementation of PID-CBF, MPC-CBF, CLF-CBF-QP, and the proposed universal formula controllers. To facilitate a comprehensive comparison, we provide both 3-D and 2-D visualizations of the tracking performance. As we see, these controllers generate trajectories that effectively avoid obstacles while tracking the desired trajectory $\mathbf{r}_{d}$. To quantify the comparisons, we 
define $\mathbf{e}_{xy}=\sqrt{\mathbf{e}_{x}^{2}+\mathbf{e}_{y}^{2}}$, which represents the tracking error in the $\mathrm{X-Y}$ direction. Then, $\mathbf{e}_{z}$ indicates the tracking error in the $\mathrm{Z}$ direction. To assess safety requirements, we analyze the quadrotor's minimal distance to obstacles using the following equation. 
\begin{equation*}
    h_{xy}=\min_{l\in\{1,\cdots,p\}}(h_{l}(x,y),\quad h_{z}=\min_{l\in\{1,\cdots,p\}}(h_{\mathrm{a},l}(z)).
\end{equation*}

Fig.~\ref{Performance_Comparison} presents a quantitative assessment of the safe tracking performance of the four controllers. Notably, all four algorithms exhibit commendable tracking performance in both $\mathrm{X-Y}$ and $\mathrm{Z}$ directions, and tracking errors tend to converge to zero if the obstacles are not encountered. However, MPC-CBF, as depicted in Fig.~\ref{Performance_Comparison}, showcases a superior tracking performance, as a benefit of its predictive capabilities. Nevertheless, it is important to acknowledge that this achievement comes at the cost of greater computational demands. In terms of safety, as indicated by $h_{xy}>0$ and $h_{z}>0$, we consistently maintain the predefined safety constraints, both in altitude and lateral $\mathrm{X-Y}$ control. Furthermore, in Fig.~\ref{Control_Input_Comparison}, we exhibit and compare the control inputs generated by these controllers. The control limits are set as follows:
\begin{equation}
    \begin{split}
       0 &=f_{\min}\leq f\leq f_{\max}=8.3385 \mathrm{N},\\
        -1.31 \mathrm{rad}&=\phi_{\min}\leq \phi\leq \phi_{\max}=1.31 \mathrm{rad},\\
        -1.31 \mathrm{rad}&=\theta_{\min}\leq \theta\leq \theta_{\max}=1.31 \mathrm{rad}.
    \end{split}
\end{equation}
It is evident that the control input is a continuous function that remains within the predefined control limits.

While the safe tracking performances are similar among all four algorithms, a notable difference arises in their computational speeds. To offer a fair comparison of these computational speeds, we conducted 100 trials and recorded the time required to achieve the final safe tracking results. Remarkably, the average execution time of the universal formula stands is approximately of only $1.31\mathrm{s}$, in contrast to the $29.95\mathrm{s}$ needed for CLF-CBF-QP and MPC-CBF (with a staggering $592.35\mathrm{s}$ for the latter). It is worth noting that PID-CBF exhibits a comparable computation speed to the universal formula, with a runtime of $1.401\mathrm{s}$. However, the distinct advantage of the universal formula lies in its provision of theoretically guaranteed stabilization. This observation strongly highlights the suitability of the universal formula for onboard implementation, particularly in scenarios where computational resources are limited.

\subsection{Experimental Tests}
Contrary to simulations, implementing an optimization-based algorithm onboard such as the Crazyflie 2.1 platform has posed significant challenges due to the limited computation resources. Therefore, we only compare the performance of the universal formula with the PID-CBF method for the experiment. The experimental setup is illustrated in Fig.~\ref{Environment_Exp}. We evaluate the algorithm using a safety-oriented tracking scenario, wherein a single drone is tasked with tracking a circular reference trajectory while avoiding three stationary obstacles. It's important to note that real-world settings inherently introduce disturbances. Consequently, we exclusively employed the universal formula, i.e., the results introduced in Section~\ref{Universal_with_Disturbance}, which accounts for disturbances in this context. 
\begin{figure}[tp]
 \centering
    \makebox[0pt]{%
    \includegraphics[width=3.5in]{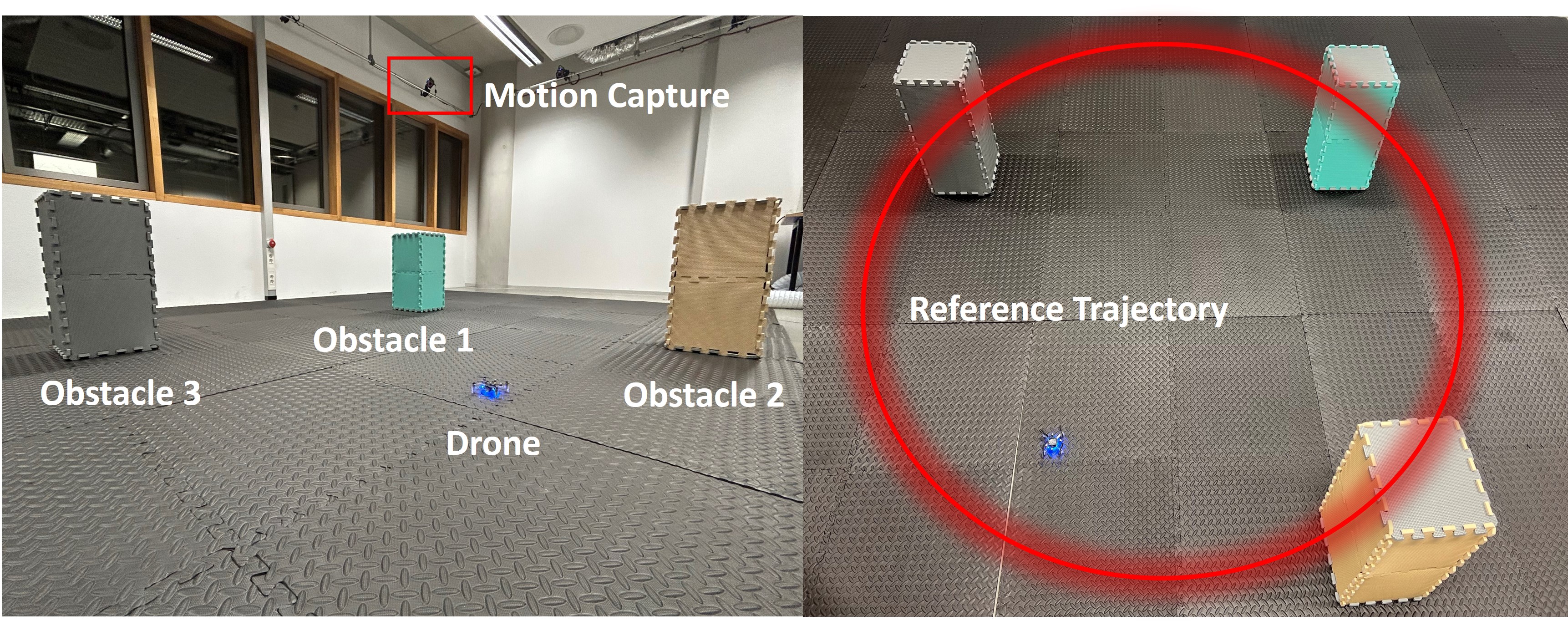}}
    \caption{Experiment: Left - Lateral view with three obstacles and one drone in the environment, using motion capture for a real-time object position. Right - Top-down view of the drone following a circular reference trajectory (red-shaded circle) while avoiding three obstacles.}
    \label{Environment_Exp}
\end{figure}

In Fig.~\ref{Ex_Safe_Tracking_Performance}, a successful safe stabilization by the proposed universal formula is achieved. Next, similar to our simulation results, a quantitative evaluation of the PID-CBF and universal formula approach is presented in Fig.~\ref{Ex_Quantified_Performance}. This analysis reveals that both methods perform equally well in a real-world experiment context. They both exhibit error convergence in the case that no obstacles are encountered, while consistently upholding safety guarantees. However, as claimed in Section~\ref{Universal_with_Disturbance}, we can only guarantee the forward invariance of a smaller safety set $\mathcal{C}_{\bm{\omega}}$ when disturbances are intruded, along with the convergence of a neighboring region of the origin. This fact can be explained from the observation that the curves $\mathbf{e}_{xy}$, $\mathbf{e}_{z}$, $h_{xy}$, and $h_{z}$ consistently maintain a non-zero distance from zero.

As for the control input, we set the parameters as follows.
\begin{equation}
    \begin{split}
       0 &=f_{\min}\leq f\leq f_{\max}=0.5 \mathrm{N},\\
        -0.5 \mathrm{rad}&=\phi_{\min}\leq \phi\leq \phi_{\max}=0.5 \mathrm{rad},\\
        -0.5 \mathrm{rad}&=\theta_{\min}\leq \theta\leq \theta_{\max}=0.5 \mathrm{rad}.
    \end{split}
\end{equation}
Finally, the control input is visualized in Fig.~\ref{Ex_Control_Input_Comparison}. Notably, the control input consistently lies in the range of the control limit, highlighting the robustness and reliability of our approach in real-world scenarios.
\begin{figure*}[tp]
 \centering
    \makebox[0pt]{%
    \includegraphics[width=5.5in]{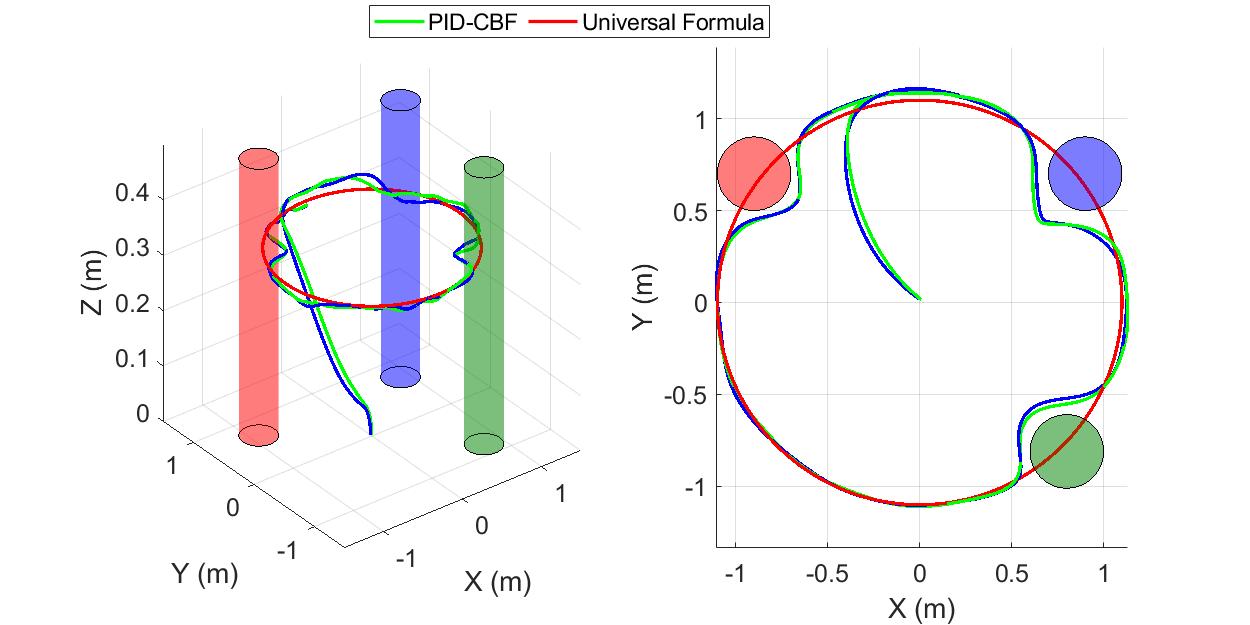}}
    \caption{Experiment - Comparative analysis of safe tracking performance: PID-CBF vs. Universal Formula method.}
    \label{Ex_Safe_Tracking_Performance}
\end{figure*}
\begin{figure*}[tp]
 \centering
    \makebox[0pt]{%
    \includegraphics[width=5.5in]{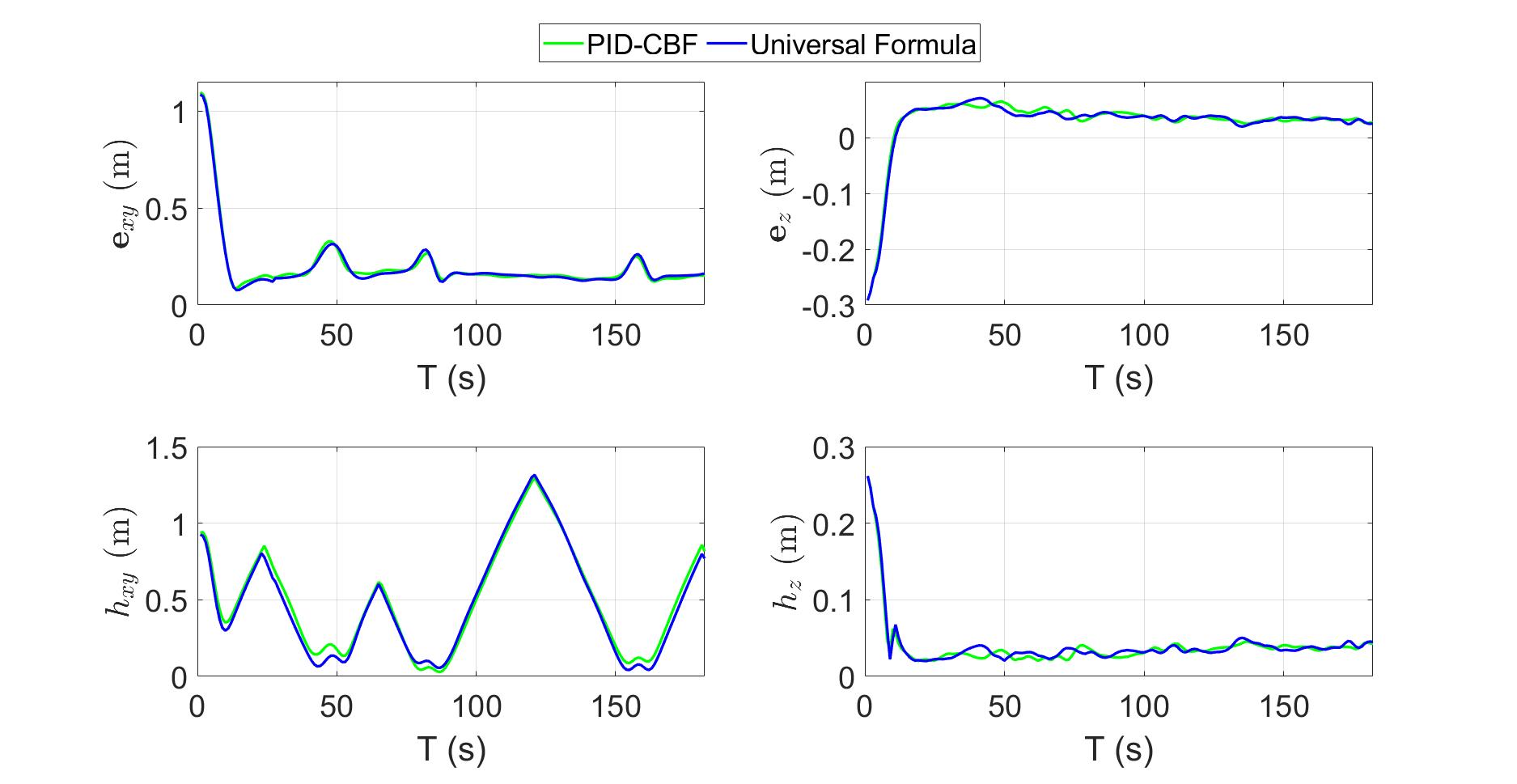}}
    \caption{Expereiment - Tracking errors and safety evaluations of four algorithms in ensuring safe tracking in both $\text{X-Y}$ and $\text{Z}$ Axes.}
    \label{Ex_Quantified_Performance}
\end{figure*}
\begin{figure*}[tp]
 \centering
    \makebox[0pt]{%
    \includegraphics[width=5.5in]{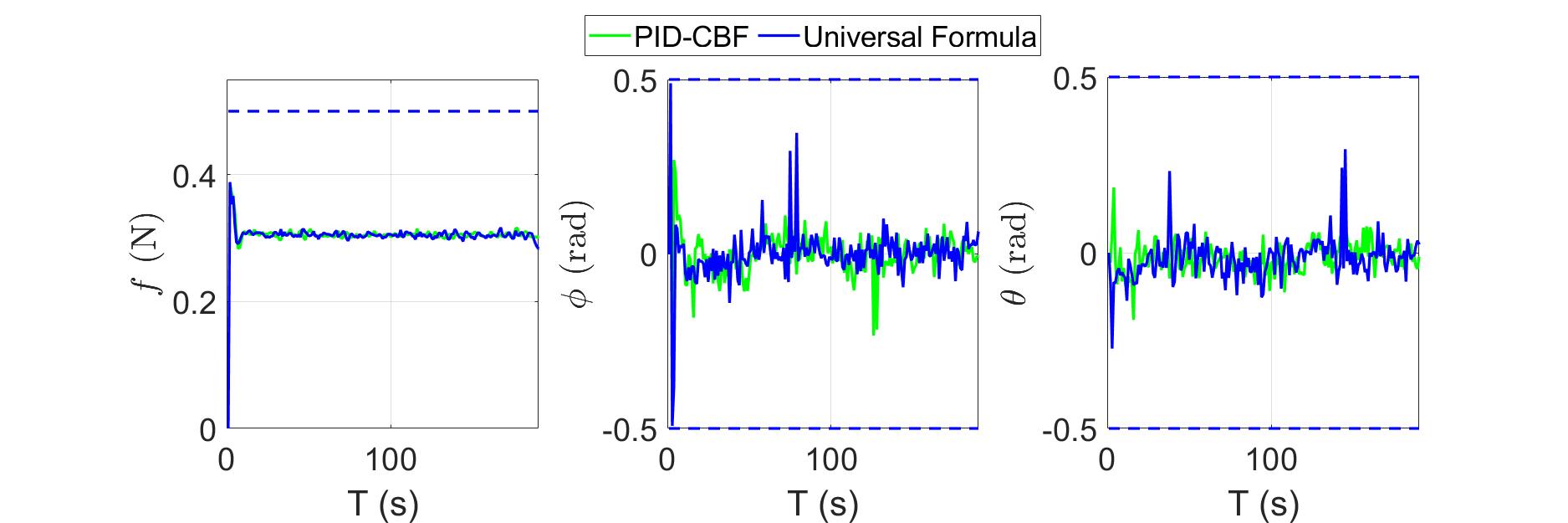}}
    \caption{Experiment - Control input behavior of two different algorithms: PID-CBF vs. Universal Formula method.}
    \label{Ex_Control_Input_Comparison}
\end{figure*}
\section{Conclusion}
This paper introduces a universal formula that leverages CLF and CBF conditions to address safe stabilization problems for quadrotors, both in cases of compatibility and incompatibility. This formula offers an attractive alternative to optimization-based approaches, eliminating the need for solving on-board optimization. When facing scenarios involving a single CLF constraint but multiple CBFs, we transform the safe stabilization problem by approximating multiple CBFs with a single, albeit more conservative CBF condition, which still yields commendable performance. Furthermore, we enhance our universal formula by incorporating the theories of ISS and ISSf, enabling the accommodation of input disturbances. This modification results in a more robust safe stabilization control while maintaining a rapid computational speed and is suitable for on-board implementation. 
Moreover, in order to effectively address input constraints, we introduce a projection-based strategy. This method involves aligning the control input, computed through the universal formula, with the closest point available within the control input domain, thereby guaranteeing compliance with the predefined constraints. Our approach is validated through both simulations and real-world experiments, demonstrating its effectiveness and highlighting its fast computation advantage. The results confirm that our solution exhibits significantly faster execution, making it well-suited for on-board implementation while maintaining a high level of safe tracking performance.
\bibliographystyle{IEEEtran}
\bibliography{Universal_Drone}
\end{document}